\documentclass[10pt]{article} 
\usepackage[preprint]{tmlr}

\usepackage[hang,flushmargin]{footmisc}
\usepackage{amsmath,amssymb,bm}
\usepackage{amsthm}

\newtheorem{lemma}{Lemma}
\newtheorem{proposition}{Proposition}

\newcommand{\M}{\mathcal{M}}
\newcommand{\R}{\mathbb{R}}
\newcommand{\E}{\mathbb{E}}

\DeclareMathOperator{\tr}{tr}
\DeclareMathOperator{\vecop}{vec}

\newcommand{\x}{\bm{x}}
\newcommand{\y}{\bm{y}}
\newcommand{\q}{\bm{x}_\mathrm{test}}
\newcommand{\task}{\bm{w}}
\newcommand{\e}{\bm{\varepsilon}}
\newcommand{\vm}{\bm{m}}
\newcommand{\vk}{\bm{k}}
\newcommand{\vv}{\bm{v}}

\newcommand{\iid}{\sim_{\mathrm{i.i.d.}}}
\usepackage{physics}
\newcommand{\Rtr}{R_\mathrm{tr}}
\newcommand{\btr}{\bm{b}_\mathrm{tr}}
\newcommand{\G}{\mathcal{G}}

\usepackage{hyperref}
\usepackage{url}
\hypersetup{colorlinks=true,allcolors=Periwinkle}
\usepackage[dvipsnames]{xcolor}
\usepackage{wrapfig}
\usepackage{graphicx}

\usepackage[most]{tcolorbox}
\definecolor{PlotPurple}{HTML}{6D5DF2}
\definecolor{PlotPurpleDark}{HTML}{635CB8}
\definecolor{PlotBlue}{HTML}{5E67D8}
\definecolor{PlotCyan}{HTML}{36B6D0}
\definecolor{PlotGreen}{HTML}{3FBE73}
\definecolor{AxisGrey}{HTML}{A0AAB8}
\definecolor{PanelGrey}{HTML}{F7F7FA}
\tcbset{
  resultbox/.style={
    enhanced,
    breakable,
    sharp corners,
    boxrule=0.55pt,
    colframe=AxisGrey,
    colback=PanelGrey,
    opacityback=1,
    boxsep=0pt,
    left=11pt,
    right=11pt,
    top=9pt,
    bottom=9pt,
    before skip=10pt,
    after skip=10pt,
    borderline west={1.6pt}{0pt}{PlotPurple},
    coltitle=PlotPurpleDark,
    fonttitle=\bfseries,
    detach title,
    before upper={
      \noindent{\color{PlotPurple}\bfseries\tcbtitle}\par
      \vspace{0.35em}
      {\color{AxisGrey!45}\hrule height 0.35pt}
      \vspace{0.75em}
    },
  }
}
\newtcbtheorem{lem}{Lemma}
  {resultbox}
  {lemma}
\newtcbtheorem{prop}{Proposition}
  {resultbox}
  {prop}
\usepackage[framemethod=TikZ]{mdframed}
\mdfdefinestyle{proofbarstyle}{
  skipabove=10pt,
  skipbelow=10pt,
  leftline=true,
  rightline=false,
  topline=false,
  bottomline=false,
  linewidth=2pt,
  linecolor=AxisGrey!35,
  backgroundcolor=white,
  innerleftmargin=10pt,
  innerrightmargin=0pt,
  innertopmargin=0pt,
  innerbottommargin=0pt,
  leftmargin=12pt,
  rightmargin=0pt,
}

\title{Sequential Correlations Change In-Context Learning: \\ Effective Context Length and Architectural Mismatch}

\author{\name Mary Letey \email maryletey@fas.harvard.edu \\
      \addr John A. Paulson School of Engineering and Applied Sciences, Harvard University \\
      Kempner Institute for the Study of Natural and Artificial Intelligence, Harvard University
      \AND
      \name Yue M. Lu \email yuelu@seas.harvard.edu \\
      \addr John A. Paulson School of Engineering and Applied Sciences, Harvard University
      \AND
      \name Cengiz Pehlevan\thanks{CP and JZ-V jointly supervised this work.} \email cpehlevan@seas.harvard.edu
      \\
      \addr John A. Paulson School of Engineering and Applied Sciences, Harvard University  \\
      Kempner Institute for the Study of Natural and Artificial Intelligence, Harvard University
      \AND
      \name Jacob Zavatone-Veth\footnotemark[1] \email jzavatoneveth@fas.harvard.edu
      \\
      \addr Society of Fellows and Center for Brain Science, Harvard University
      }

\def\month{06}  
\def\year{2026} 

\begin{document}

\maketitle

\begingroup
\renewcommand{\thefootnote}{}
\footnotetext{Code is available at \url{https://github.com/Pehlevan-Group/sequential-correlations-in-context-regression}.}
\footnotetext{A previous version of this work was presented at the \href{https://sites.google.com/view/hidimlearning/home}{4\textsuperscript{th} Workshop on High-dimensional Learning Dynamics (HiLD)} at ICML 2026 \citep{letey2026sequential}.}
\endgroup

\begin{abstract}
Modern sequence models have a striking capacity for in-context learning (ICL); they can perform new tasks based only on examples given in the prompt. Understanding how this ability emerges requires theory that captures important properties of natural data. Linear regression has served as a useful sandbox for ICL theory, but existing work has largely focused on prompts with independent examples. In this work, we extend this setting to sequentially correlated data, a basic feature of real sequences. We present a solvable model based on linear attention and test our predictions on realistic transformer architectures. We identify two distinct effects: First, when the query token is independent of the context, within-context correlations induce an effective context length: correlated prompts behave like shorter i.i.d. prompts. Second, when the query is also correlated with its context, test error is reduced, particularly for softmax attention when compared to linear attention. These results suggest that correlated prompts alter not only the effective sample size of in-context learning, but also which attention architectures are best matched to the task.
\end{abstract}

\section{Introduction}
In-context learning (ICL) is an important and useful capability of modern sequence models \citep{brown2020language, vonoswald2023transformers, wei2022emergent}. In this setting, a model performs a task implicitly from few examples without parameter updates. Regression has been a useful toy setting for studying how sequence models can achieve this ability. In-context regression requires a model to predict the label of a final query based on a prompt of input-output examples. The appeal of this setup is that it is simple enough to admit explicit analyses, yet rich enough to differentiate architectures and learning mechanisms \citep{akyurek2023what,vonoswald2023transformers,zhang2023doesincontextlearninglearn,lu2025asymptotictheoryincontextlearning,zhang2024incontext, oko2024pretrainedtransformerefficientlylearns,vasudeva2025transformers,letey2025pretrain}.

A standard simplifying assumption in the in-context regression literature is that the examples within a prompt are independent. This assumption is analytically convenient, but not reflective of real datasets \citep{shannon1948communication}. This has prompted the study of sequential correlations in other, non-regression in-context data settings, including Markov chains and dynamical systems \citep{edelman2024evolutionstatisticalinductionheads, nichani24a, cole2025incontextlearninglineardynamical, park2025competitiondynamicsshapealgorithmic}. For the regression setting, the effects of sequential correlations have yet to be analysed. We know that sequential correlations matter even for classical ridge regression, where they change the asymptotic risk and invalidate estimators designed for i.i.d.\ data \citep{atanasov2025risk}. Thus, it is natural to ask how sequential correlations affect in-context regression, particularly as attention architectures were designed to operate on structured sequences \citep{vaswani2017attention}.

In this work, we study a solvable model of in-context linear regression with sequentially correlated tokens. We begin from a reduced linear-attention model analysed in prior works \citep{zhang2023trained, wu2023pretraining, lu2025asymptotictheoryincontextlearning, letey2025pretrain, bordelon2025theoryscalinglawsincontext}, and perturb the usual i.i.d.\ setup by introducing a sequence correlation kernel for the context tokens, together with an optional correlation between the test query and the preceding context. This gives a controlled setting in which we can separate the effects of correlations \emph{within} the context, and correlations \emph{between} the query and the context.

\begin{figure}[t]
\centering
\includegraphics[width=\textwidth]{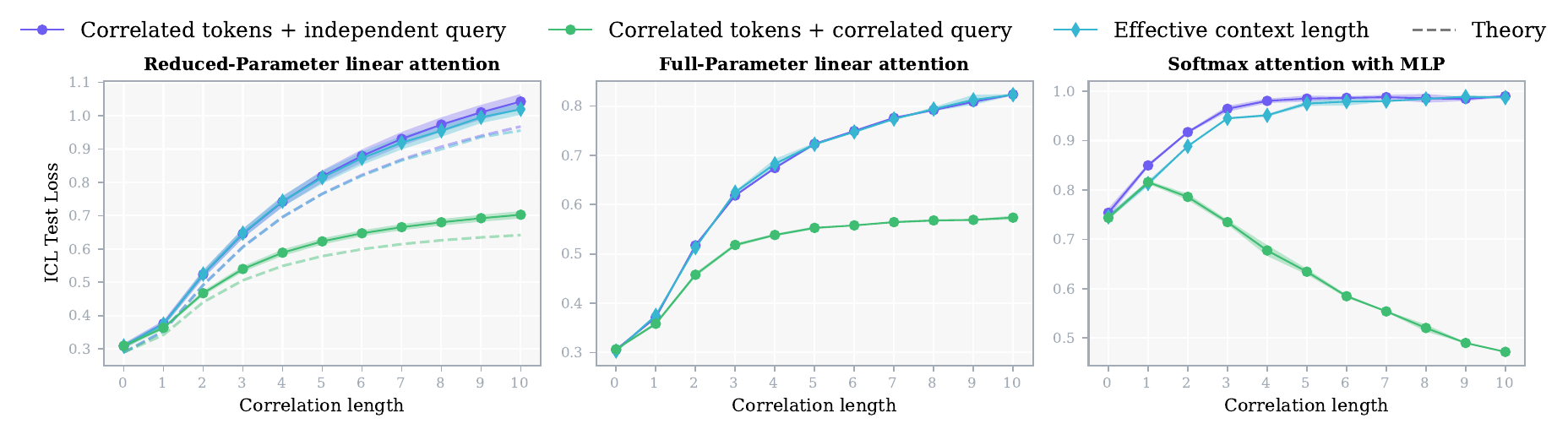}
\caption{ICL test loss against token correlation strength for three attention architectures. We highlight two findings. (1) ICL performance on correlated data with an independent query \textcolor{Periwinkle}{(purple)} is well approximated by uncorrelated data at an effective context length
\textcolor{CornflowerBlue}{(blue)} given by Eq. \eqref{eq:eff_len}. (2) When the query token is equivalently correlated with its corresponding context \textcolor{ForestGreen}{(green)}, ICL error is reduced in all models as they utilise this correlation for inference. These predictions are made from our theory curves \textcolor{gray}{(dashed)} derived for reduced-parameter linear attention. Experimental details given in Appendix \ref{sec:experiments}.}
\label{fig:compare}
\end{figure}

Figure \ref{fig:compare} summarises our findings. First, when correlations are confined to the context and the query remains independent, their effect is equivalent to an effective context-length reduction: correlated prompts behave like shorter i.i.d.\ prompts. Second, when the query token is correlated with its preceding context, ICL error decreases substantially, because the model can exploit this extra statistical structure to improve inference. Third, these correlated settings reveal an architectural mismatch: when additional query-dependent statistics enter, the performance gap between attention architectures (linear versus softmax in Figure \ref{fig:compare}) widens sharply.

An additional technical contribution is that we identify a regime in which strong sequential correlations lead to the breakdown of the high-dimensional treatment originally formulated in \cite{lu2025asymptotictheoryincontextlearning}. When correlations persist on a length scale that grows with the context length, the effective number of independent examples may no longer scale proportionally with dimension, and the moments of the linear attention estimator need not concentrate. 

\subsection*{Related Works}
Here we summarise key empirical and theoretical results from previous studies of in-context learning. The techniques and data structures considered in these past works motivate our study of in-context regression with sequential correlations. 

\paragraph{ICL in practice.} Early works showed that language models demonstrate an ability to learn new tasks, without weight updates, given a few examples provided ``in context'' \citep{radford2019language,brown2020language, chowdhery2022palmscalinglanguagemodeling, touvron2023llama, achiam2023gpt, bubeck2023sparksartificialgeneralintelligence}. These phenomena seemed to emerge with scale \citep{wei2022emergent, wei2023largerlanguagemodelsincontext} and to be robust to task complexity \citep{wei2022_chainofthought,deepseekai2025}. These observations led to much interest in understanding ICL \citep{dong2024surveyincontextlearning}, particularly regarding the question of how unknown tasks are solved at test time. Many empirical works attempted to answer this question, focusing on mechanisms \citep{vonoswald2023transformers,li2024languagemodelslearncontext, akyurek2023what, akyurek2024incontextlanguagelearningarchitectures, kirsch2024generalpurpose, olsson2022context, tong2024mlpslearnincontext, finegrained_data_arch_oymak} and data structures \citep{bietti2023birth,
chan2022data,
raventos2023pretraining,
garg2022what,
liu2025probingincontextlearningimpact,
singh2023transient,
goddard2025incontextlearninggeneralizetask,
Liu_2024,
vasudeva2025transformers,todd2026incontextalgebra} that enable ICL. 

\paragraph{ICL in theory.} Recent theoretical works have sought to answer similar questions about in-context learning, specifically, what algorithms could it be implementing, and what is the effect of data on its performance? Simple models, such as one- or two-layer networks and linear attention, have proven particularly useful testbeds for understanding what happens at inference time \citep{chandra2024towards, fu2023transformers,shen2025understandingincontextlearningstructured,gozeten2025testtimetrainingprovablyimproves,mahankali2023stepgradientdescentprovably,barnfield2026multilayercrossattentionprovablyoptimal,ahn2023transformers,zhang2025trainingdynamicsincontextlearning,sametDemir_2025_datamixing}. In particular, many theoretical works have been able to show through constructive approaches that transformers or simple architectures are able to implement general algorithms, such as gradient descent, to solve tasks in-context \citep{bai2023transformers, reddy2023mechanistic, elmoznino2025incontextlearningoccamsrazor, lee2025distinct,zhang2023doesincontextlearninglearn}. On the dataset side, many theoretical works have explicitly considered sequential structures, such as dynamical systems \citep{li2023transformers, cole2025incontextlearninglineardynamical} and Markov chains \citep{edelman2024evolutionstatisticalinductionheads,park2025competitiondynamicsshapealgorithmic,nichani24a}, in addition to more complex synthetic tasks \citep{tao2025howtransformerspredictpseudorandom,learning_to_grok_doshi_24,ye2026transformersprovablylearnalgorithmic,cole2026incontextoperatorlearningspace}. 

\paragraph{ICL in regression.} Moving away from complex, reasoning, and language-like tasks, regression has been a very rewarding sandbox for ICL research. By considering simpler functions, stronger results such as sample sufficiency \citep{oko2024pretrainedtransformerefficientlylearns}, gradient descent-like behaviour \citep{wu2023pretraining,zhang2023trained,zhang2024incontext,bordelon2025theoryscalinglawsincontext,fu2024transformerslearnachievesecondorder}, and even finite-sample error formulas \citep{lu2025asymptotictheoryincontextlearning,letey2025pretrain}, can be derived. This work builds on these to clearly understand the effect of sequential correlations on in-context learning of linear regression, as a probe of how different architectures are affected by such correlations. 

\section{Theory and Data Setup}\label{sec:main_paper_setup}
\textbf{In-context linear regression.} We study a correlated-token version of the solvable linear-regression ICL model induced by a reduced linear-attention block, studied by \citet{lu2025asymptotictheoryincontextlearning,zhang2023doesincontextlearninglearn}. A context consists of a sequence  
$\{ \x_1,y_1,\ldots,\x_\ell,y_\ell,\q\}$ with 
\[
y_a = \x_a^\top\task + \varepsilon_a,\qquad y_{\mathrm{test}} = \q^\top \task + \varepsilon_{\mathrm{test}},
\]
for task vector $\task\in\R^d$, and independent Gaussian noise $
\e\sim\mathcal N(0,\rho I_\ell),\, \varepsilon_{\mathrm{test}}\sim\mathcal N(0,\rho).$ The goal is for the model's output on this sequence to be close to $y_\text{test}$, thus performing the correct regression inference given examples in the context. Throughout, we write $X\in\mathbb{R}^{\ell\times d}$ for the matrix whose $a$th row is $\x_a^\top$.

\textbf{Sequential correlations.} We model sequential structure within the context as linear dependence between the tokens, following previous work in the ridge regression setting \citep{atanasov2025risk,moniri2025linearly}. Under this model, the mean-zero Gaussian tokens $\x_1\,,\ldots\,,\,\x_\ell$ have second moments given by 
\[
\E[\x_a\x_b^\top] = K_{ab}I_d/d, \qquad \E[\q\x_a^\top] = \bm{k}_\text{test}\,_a I_d/d
\]
\textit{i.e.}, $K\in\R^{\ell\times \ell}$ controls correlations across context positions, and the feature-feature covariance is given by $I_d/d$. The case of independently sampled tokens is described by $K = I_\ell$. The query correlations are controlled by $\vk_{\mathrm{test}} \in \mathbb{R}^\ell$ which specifies how strongly the query is coupled to each preceding context token; for an independent query, $\vk_\mathrm{test} = \bm{0}.$

\textbf{Correlation summaries and correlation length.} 
We normalise such that $\tr(K)=\ell$. The theory depends on $K$ and $\vk_{\mathrm{test}}$ through the summary statistics
\[
k_0:=\vk_{\mathrm{test}}^\top \vk_{\mathrm{test}},\qquad
k_1:=\vk_{\mathrm{test}}^\top K\vk_{\mathrm{test}}, \qquad k_2 := \tr(K^2)/\ell.
\]
For general $K$ and $\bm{k}_\text{test}$, we can think of $k_0, k_1, k_2$ as measures of token correlation strength. 
The quantity $k_2$ is a bulk measure of context correlations: under our chosen normalisation
$\tr(K)=\ell$, it is minimised by the uncorrelated case $K=I_\ell$, where $k_2=1$, and increases
as the spectrum of $K$ becomes more anisotropic or concentrated. The query-dependent quantity $k_0$ measures the overall strength
of the direct query-context coupling, while the mixed quantity $k_1$ measures how strongly
that query-correlation profile is amplified by the correlated structure of the context itself.

These quantities become more intuitive if we take a specific kernel as an example: the exponential kernel $K_{ab} = \exp(-|a-b|/\xi)$ and $\bm{k}_\text{test}\,_a = \exp(-(\ell+1 - a)/\xi).$ Because $K$ and $\bm{k}_\text{test}$ only depend on \textit{distances} between tokens, the natural parameter of correlation ``strength'' is correlation \textit{length} $\xi$. This case provides an intuitive sandbox, as the ICL summary statistics above are directly related to the correlation length by $
k_2, k_0 \approx \xi,\,
k_1 \approx \xi^2.$ We thus choose to use the exponential kernel for the experiments in all figures, and save a more detailed analysis with non-exponential kernels to future work.

\textbf{Reduced linear-attention predictor.}
Following prior work on linear attention, the next-token prediction for $y_\text{test}$ made by linear attention can be well-approximated by 
\[
\hat y_{\mathrm{test}}:=\tr(\Gamma H^\top),
\quad
\Gamma\in\R^{d\times (d+1)}, \quad 
H
:=
\q
\begin{bmatrix}
\frac{d}{\ell}(X\task+\e)^\top X
&
\frac{1}{\ell}(X\task+\e)^\top(X\task+\e)
\end{bmatrix}.
\]
The parameter matrix $\Gamma$ here is defined in terms of components of the value, key, and query matrices from full-parameter linear attention. This is the predictor that we will study in the theory, and is the ``reduced model'' referred to on the leftmost panel of Figure \ref{fig:compare}. 

Given a training batch of $n$ sequences, corresponding to $n$ such data matrices $H$ for the reduced model, the optimal parameters and corresponding test loss can be written as \begin{align*}
    \Gamma^* = \qty(\frac{n}{d}\lambda I_{d(d+1)} + \sum_{\mu=1}^n H^\mu \otimes H^\mu)^{-1}\sum_{\mu=1}^n y_\text{test}^\mu H^\mu\,, \quad \mathcal{E}_\text{ICL}(\Gamma^*) = \E[(y^\text{new}_\text{test} - \langle H^\text{new}, \Gamma^*\rangle)^2].
\end{align*}

\textbf{High-dimensional scaling.}
We study the matrix $\Gamma^*$, and functions of it, in a high-dimensional limit using random-matrix techniques as in \cite{lu2025asymptotictheoryincontextlearning}. We work in the standard proportional regime, where $d \to \infty$ with $\ell \propto d$. We assume the task signal remains identifiable, e.g. $
\task^\top\task = \Theta(d).$ 

\begin{wrapfigure}{r}{0.58\textwidth}
    \centering
    \includegraphics[width=\linewidth]{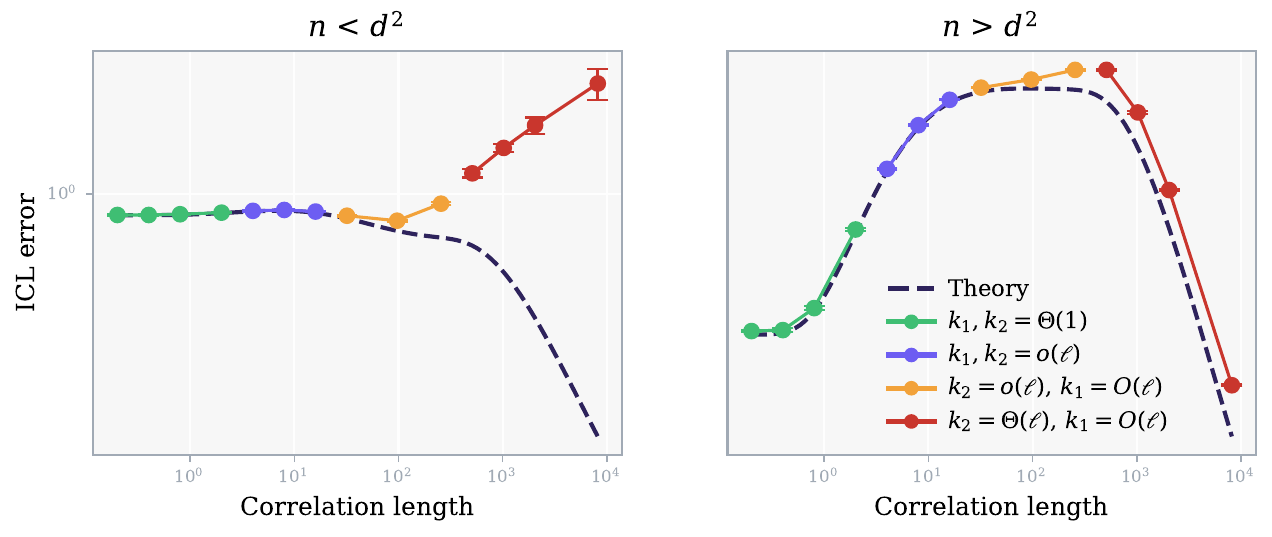}\vspace{-1em}
    \caption{Simulations of $\mathcal{E}_\mathrm{ICL}(\Gamma^*)$ for several scaling regimes, compared with theory (Prop. \ref{prop:icl-error-formula-constant}).}
    \label{fig:three_regimes}
    \vspace{-1em}
\end{wrapfigure}
We analytically study the case of \textit{weak correlations} with respect to context length $\ell$: for the computation of theory curves we take $k_2,k_0,k_1 = \Theta_\ell(1)$. We do not study stronger correlations in this work, corresponding to cases such as $k_2 = o(\ell)$ or $k_2 \propto \ell$, and do not attempt to characterise all of the ways in which the theory's predictions break down. In stronger-correlation regimes, concentration of the parameter matrix $\Gamma^*$ and its corresponding ICL error can fail (see Figure \ref{fig:three_regimes}), and a full characterisation is left for future work, with some further discussion in Section \ref{sec:effective_context_section} and Appendix \ref{sec:nonconcentration}.

\section{Effective in-context sample size}\label{sec:effective_context_section}
Working in the proportional limit described above, we can analyse the random parameter matrix $\Gamma^*$. The full result is given in Appendix \ref{sec:actual_formula}. We present its first implication.

\begin{tcolorbox}[resultbox,title={Effective context length formula}]
Consider tokens $\x_1,\ldots,\x_\ell$ that are weakly correlated, \textit{i.e}. $k_2=\Theta(1)$, and an independent query $\q$, \textit{i.e}. $\bm{k}_\mathrm{test}=\bm{0}$. Then the correlated-ICL performance at context length $\ell$ matches the uncorrelated-ICL performance at the lower effective context length \begin{equation}\label{eq:eff_len} \ell_\text{eff} := \frac{1+\rho}{k_2+\rho}\ell,\end{equation} where by Jensen's inequality $k_2 = \tr(K^2)/\ell \geq \qty(\tr(K)/\ell)^2 = 1$, so $\ell_\mathrm{eff} \leq \ell.$
\end{tcolorbox}

Eq (\ref{eq:eff_len}) makes precise how sequential correlations reduce the usable information content of the prompt. Since $k_2$ increases as the context covariance
becomes more anisotropic, stronger or more persistent correlations decrease the effective context length $\ell_{\mathrm{eff}}$. For stationary kernels, where $k_2$ grows with the correlation scale, this recovers the intuitive statement that longer-ranged correlations make fewer context examples effectively independent. The simplicity of the resulting formula is itself notable: the reduction depends linearly on the single scalar $k_2$, rather than on the full
spectrum of $K$. This contrasts with classical correlated linear regression, where test error typically depends more delicately on the spectrum, and sometimes the eigenvectors, of the design covariance.

Eq (\ref{eq:eff_len}) also hints why the strongly-correlated regime is analytically difficult.
Our theory here assumes that the correlation summary $k_2$ remains $\Theta(1)$ as
$\ell,d\to\infty$ with $\ell\propto d$. But once correlations persist across more and more of the context, we instead have either $k_2=o(\ell)$ or $k_2 \propto \ell$. In this regime, $\ell_{\mathrm{eff}}$ no longer scales proportionally with $\ell$, but remains only $o(\ell)$ or $\Theta(1)$ respectively. Thus, even as the nominal context length increases, we are no longer in the standard linear-regression regime where samples are proportional to dimension. We discuss this further in Appendix \ref{sec:nonconcentration}, as these regimes are where the standard high-dimensional concentration underlying our asymptotic analysis may break down.

\section{Gain from query correlations}
\begin{wrapfigure}{r}{0.5\textwidth}
    \centering
    \vspace{-0.4cm}
    \includegraphics[width=0.5\textwidth]{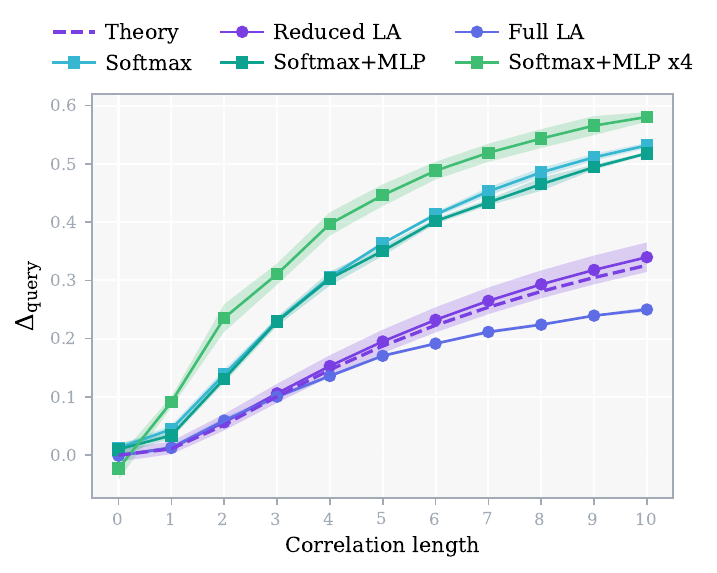}
    \caption{Query effects depend on architecture.}
    \label{fig:delta_query}
\end{wrapfigure}

We now turn to the second effect in Figure \ref{fig:compare}: correlations between the test query and its preceding context can provide additional predictive signal and thus reduce the ICL error. To isolate this contribution, define the \emph{query-correlation gain} as
\[
\Delta_{\mathrm{query}}
:=
\mathcal{L}_{\mathrm{ICL}}(\ell, k_2\,; \, \bm{k}_\text{test}=0) - \mathcal{L}_\text{ICL}(\ell, k_2, k_0, k_1),\]
namely, the reduction in ICL error obtained by introducing query-context correlations while holding the bulk context kernel $K$ fixed. Here $\mathcal{L}_\mathrm{ICL}$ denotes the ICL error of a model (not necessarily linear attention) on this correlated linear regression task.

In Figure \ref{fig:delta_query}, we find that $\Delta_{\mathrm{query}}$ is nonnegative (up to finite-sample fluctuations) and vanishes in the uncorrelated case $K=I_\ell$. Thus, unlike bulk context correlations, which reduce effective sample size and hurt ICL performance, query correlations provide additional predictive signal that can in principle be exploited for inference.

In our theory, this effect is controlled not only by the bulk statistic $k_2$, but also by the query-dependent summaries $k_0$ and $k_1$, which increase with stronger correlations. The resulting closed-form expression for $\Delta_{\mathrm{query}}$ is substantially more complicated than the effective-sample law in eq. (\ref{eq:eff_len}); we therefore omit it from the main text. Empirically, $\Delta_{\mathrm{query}}$ is relatively modest for linear-attention architectures but much larger for architectures that use softmax attention. This suggests that softmax attention is better matched to sequential inference, as it can form query-dependent weights over the most informative tokens rather than compressing correlations into the second-order summaries available to linear attention.

\section{Conclusion} 
This work is a first step toward a theory of ICL for structured, sequential data. By separating bulk context correlations from query-context correlations, we show that prompt structure can affect ICL in two qualitatively different ways: bulk context correlations reduce effective sample size, while query-context correlations provide useful signal for inference, highlighting an architectural mismatch between linear and softmax attention. The latter effect is especially important: softmax attention appears to have a distinctive ability to exploit sequential information, which is not captured by the linear-attention theory. Developing theory that explains this advantage, and extending it beyond linear regression to richer sequential and dynamical tasks, will be necessary for understanding ICL in more realistic data regimes.

\section*{Acknowledgements}
M.L. is supported by a Graduate Fellowship from the Kempner Institute for the Study of Natural and Artificial Intelligence. Y.M.L. gratefully acknowledges support from a Harvard College Professorship, the Harvard FAS Dean’s Fund for Promising Scholarship, and DARPA grant DIAL-FP-038. C.P. was supported by an NSF CAREER Award (IIS-2239780), DARPA grants DIAL-FP-038 and AIQ-HR00112520041, the Simons Collaboration on the Physics of Learning and Neural Computation, and the William F. Milton Fund from Harvard University. J.Z.-V. was supported by a Junior Fellowship from the Harvard Society of Fellows. This work has been made possible in part by a gift from the Chan Zuckerberg Initiative Foundation to establish the Kempner Institute for the Study of Natural and Artificial Intelligence.

\newpage
\bibliographystyle{tmlr}
\bibliography{refs}

\newpage
\appendix
\section{Experimental details}\label{sec:experiments}
The data structure follows the Section \ref{sec:main_paper_setup} setup. For all experiments we choose a stationary kernel, \textit{i.e.}, distance-dependent correlations, given by \begin{equation}\label{eq:exponential_bulk_kernel} K_{ab} := \exp(-|a-b|/\xi).\end{equation} Here the correlation length $\xi$ is manifestly part of the definition of $K$; when $\xi = 0$, $K = I_\ell$. This definition is convenient due to the clarity of its correlation length parameter, as well as the fact that our Gaussian model with this kernel is equivalent to an AR(1) process which allows for efficient sampling. When query correlations are present, \textit{i.e.}, $\bm{k}_\text{test} \neq \bm{0}$, we choose matching structure \begin{equation}\label{eq:exponential_query_kernel}(\bm{k}_\text{test})_a := \exp(-|\ell+1-a|/\xi).\end{equation} This parameter $\xi$ is what is referred to on the $x$-axes of all figures.

\paragraph{Data parameters.} A batch of data contains sequences of length $\ell$ with one query each. We train in an offline manner with $n$ total sequences. Each sequence is defined by a task vector $\task$. As in \cite{lu2025asymptotictheoryincontextlearning, letey2025pretrain, raventos2023pretraining}, the total number of task vectors in the pretraining batch may differ from $n$. The total number of unique task vectors in the batch is called $k$, and each unique vector is sampled i.i.d. from $\mathcal{N}(\bm{0}, I_d).$ For all experiments, we take $$d = 32\,, \quad n = 4096 = 4d^2\,, \quad \ell=128=4d\,, \quad k = 320 = 10d$$ matching the asymptotic scalings of these variables derived in \cite{lu2025asymptotictheoryincontextlearning}. We take label noise $\rho = 0.01$. 

\paragraph{Architecture details.} The experiments train sequence models on the above linear regression sequence data with sequentially correlated inputs. The architectures considered include softmax-only layers, softmax+MLP transformer-like architectures, and pure linear attention. Note that we distinguish between the reduced model (where the optimal solution is given directly by $\Gamma^*$) from \textit{trainable} fully-parameterised linear attention. Softmax models embed inputs into width $d+1$ with a dedicated label channel, initialise input embeddings with scale proportional to $\sqrt{(d+1)}$, initialise query/key maps on the input channels, and initialise value/output maps to primarily use the label channel with small Gaussian noise. MLP blocks use GELU and are initialised close to zero. Linear-attention models operate directly on concatenated $(\bm{x}_i, y_i)$ tokens of dimension d+1, with query/key/value matrices initialised near identity. Models are trained by minimising mean-squared error on the held-out query label only. Optimisation uses Optax AdamW-style updates with weight decay \texttt{lamb}; the total number of gradient steps is set to some fixed $T$ (large enough to reach training loss convergence), but we employ early stopping and report the test loss of the best checkpointed model. We use a linear learning-rate warmup from 0 to \texttt{max\_lr} over the first 10\% of training, with learning rate remaining constant after warmup. We use exponential moving average weights for evaluation. Minibatching is used to divide the full dataset of $n$ sequences; we find that smaller batch sizes perform better for softmax attention architectures. Figures \ref{fig:compare} and \ref{fig:delta_query} show mean best-test loss over different seeds, where seeds control train data batch, testing samples, and batch divisions. A summary of our training hyperparameters is $$\texttt{max\_lr} = 0.001, \quad \texttt{lamb}=0.0001, \quad T=10000, \quad \text{batch size} = 64 = 2d, \quad \text{\# seeds}=5.$$

\section{Detailed setup}
Here we reiterate the data and model setup given in the main document with more details. We leave the token design matrix general, \textit{i.e.}, $$\x_a \sim_K \mathcal{N}(0,\Sigma/d)$$ where the sample-sample correlations are given by positive semi-definite $K$ as \begin{equation}\label{eq:secondmoment} \E_X[X_{ai}X_{bj}] = \frac{1}{d}K_{ab}\Sigma_{ij} \qquad \text{for } a,b\in[\ell] \text{ and } i,j\in[d].\end{equation} The label noise we take to be sequentially uncorrelated   
$$\e \sim \mathcal{N}(0,\rho I_{\ell+1}) \qquad \text{independent of all $X$ and $\q$}.$$ 
The query $\q$ may be correlated with the contexts as \begin{align*} 
\E[\q |X] = X^\top \bm{m}\,, \qquad \text{Var}(\q |X) = \frac{1}{d} (1-k_{-1})\Sigma
\end{align*} for   \begin{align*}
    \bm{m} \equiv K^{-1} \bm{k}_\mathrm{test} \in\mathbb{R}^\ell\,, \qquad k_{-1} \equiv \bm{k}_\mathrm{test}^\top K^{-1} \bm{k}_\mathrm{test} \leq 1.
\end{align*} This is equivalent to sampling, as in eq. (\ref{eq:secondmoment}), using $\ell+1 \times \ell+1$ positive-semidefinite correlation kernel $$K_\text{query} = \begin{bmatrix}
    K & \bm{k}_\text{test} \\
    \bm{k}_\text{test}^\top & 1 
\end{bmatrix}.$$ As a final note, this entire data setup is equivalent to $$X = \frac{1}{\sqrt{d}}\sqrt{K} Z \sqrt{\Sigma}\,, \qquad Z_{si} \sim_\mathrm{i.i.d.} \mathcal{N}(0,1).$$

\subsection{Notation and Assumptions} 
We will consider a standard high-dimensional proportional limit, as in \citep{letey2025pretrain,lu2025asymptotictheoryincontextlearning}. The context length scales as $$\alpha \equiv \frac{\ell}{d} = \Theta_{d,\ell}(1).$$
The usual high-dimensional assumptions on $\Sigma$ are such that $$\tr(\Sigma),\, \tr(\Sigma \task\task^\top) = \Theta(d).$$ These are a succinct summary of conditions on both $\Sigma$ and $\task$. The first condition $\tr(\Sigma) = \Theta(d)$ ensures that there is reasonable signal to estimate in the tokens; note that after this section we will always assume $\Sigma = I_d$. The second condition $\tr(\Sigma\task\task^\top) = \Theta(d)$ is equivalent to $\|\task\|^2 = \Theta(d)$ ensuring that there is enough recoverable task signal. Given these assumptions, everything will be written in normalised $\Theta_d(1)$ quantities \begin{align*}
    \tr[\Sigma] := \frac{1}{d}\tr(\Sigma)\,, \qquad \tr[\Sigma \task\task^\top] := \frac{1}{d}\tr(\Sigma \task\task^\top).
\end{align*} 
These need to be extended to assumptions about $K$. We will take \begin{align*}
    \tr[K] = \frac{1}{\ell}\tr(K) = 1
\end{align*}
The final thing to reason about is the contributions from $K^2$ and from $\bm{k}_\mathrm{test}$. This is where the notion of ``correlation length'' captured in $K$ and $\bm{k}_\mathrm{test}$ becomes important, as the terms \begin{align*}
    \tr(K^2)\,, \quad k_{-1}\equiv \bm{k}_\mathrm{test}^\top K^{-1}\bm{k}_\mathrm{test} \,,\qquad k_0 \equiv \bm{k}_\mathrm{test}^\top \bm{k}_\mathrm{test}\,,\qquad k_1 \equiv \bm{k}_\mathrm{test}^\top K \bm{k}_\mathrm{test}
\end{align*} will appear in the computation. 

\paragraph{Example: the exponential kernel} It is often convenient and natural to choose $K$ to be PSD and Toeplitz, \textit{i.e.}, $$K_{st}=c(|s-t|).$$ This means the correlations between $\x_s$ and $\x_t$ are stationary: they only depend on their relative distance in the sequence. A particularly nice example of such a matrix is one with $$c(\tau) = \exp(-\tau/\xi).$$ As we've seen from Figure \ref{fig:three_regimes}, there are various cases that may occur. 
\begin{itemize}
    \item Finite correlations: $\xi = \Theta_\ell(1).$ Then $k_0,k_1,k_2$ = $\Theta(1)$.
    \item Subleading correlations: Have $\xi = o(\ell)$. If $\xi^2 = o(\ell)$, then $k_0,k_1,k_2$ = $o(\ell)$; else, have $k_0,k_2 = o(\ell)$, $k_1 = O(\ell)$.
    \item Proportional correlations: $\xi \propto \ell$. Then $k_0,k_2 = \Theta(\ell)$, $k_1 = \Theta(\ell^2)$. 
\end{itemize}

\section{Population loss}\label{sec:populations}
We will consider linear features \begin{align*}
    \bm{y} &= X\task + \e\,, \quad y_\mathrm{test} = \q^\top \task + \mathcal{N}(0,\rho).
\end{align*}
For each such sequence of $X, \y, \q, y_\mathrm{test}$, define a data matrix $$H = \q\begin{bmatrix}
    \frac{d}{\ell}({X}\task + \e)^\top {X} & \frac{1}{\ell}({X}\task + \e)^\top({X}\task + \e)
\end{bmatrix}\in \mathbb{R}^{d\times (d+1)}.$$ 
The predictor we will use for $y_\mathrm{test}$ is given by $$\hat{y}_\mathrm{test} = \langle \Gamma, H \rangle = \text{tr}\qty(\Gamma H^\top)$$ for parameters $\Gamma \in \mathbb{R}^{d\times (d+1)}.$ Before we consider ICL error at optimal parameters $\Gamma^*$, we can compute a population formula for $$\mathcal{E}_\text{ICL}(\Gamma) = \E_\text{new data}\qty[\qty(y_\text{test}^\text{new}-\langle \Gamma, H^\text{new}\rangle )^2].$$ This formula will be different depending on which asymptotic treatment of $k_2,k_0, k_1$ we are considering. 

\paragraph{Notation.} We will use the convention of row-wise vectorisation, and so have $$\text{Vec}(vu^\top) = v \otimes u \in \mathbb{R}^{\text{dim}(v)\times \text{dim}(u)}\,, \qquad \text{Vec}(vu^\top)\text{Vec}(vu^\top)^\top = (vv^\top)\otimes (\bm{u}\bm{u}^\top)$$

\begin{lemma}\label{thm:population_icl_risk_CONSTANT}
Suppose we have weak correlations, \textit{i.e.}, $k_2,k_0,k_1= \Theta(1).$ For $\task\sim \mathcal{N}(\bm{0},C)$, writing $\tr[C] = \tr(C)/d$, with $$\rho_1 \equiv \tr[C] + \rho\,, \qquad \rho_2 \equiv k_2\tr[C] + \rho$$ have \textbf{population ICL risk} \begin{align}\mathcal{E}_\mathrm{ICL}(\Gamma) = \rho + \tr[ C] -2\frac{1}{d}\tr(\Gamma A^\top) + \frac{1}{d}\tr(\Gamma B \Gamma^\top) + \frac{1}{d^2}\tr(\mathcal{T}\vecop(\Gamma)\vecop(\Gamma)^\top)
\end{align} for \begin{align*}
    A &\equiv \begin{bmatrix}
        C + \frac{1}{\alpha}k_0 \tr[C]I_d & \bm{0} 
    \end{bmatrix}\,, \qquad B \equiv \begin{bmatrix}
         C  + \frac{\rho_2}{\alpha} I_d & \bm{0} \\
        \bm{0} & \rho_1^2
    \end{bmatrix} \\
    \mathcal{T} &\equiv \frac{k_1\tr[ C] + \rho k_0}{\alpha^2}\vecop\qty(\begin{bmatrix}I & \bm{0}\end{bmatrix})\vecop\qty(\begin{bmatrix}I & \bm{0}\end{bmatrix})^\top \\
    &\quad+\; \frac{k_0}{\alpha}\qty(\vecop\qty(\begin{bmatrix}I & \bm{0}\end{bmatrix})\vecop\qty(\begin{bmatrix}C & \bm{0}\end{bmatrix})^\top + \vecop\qty(\begin{bmatrix}C & \bm{0}\end{bmatrix})\vecop\qty(\begin{bmatrix}I & \bm{0}\end{bmatrix})^\top)
\end{align*}
\end{lemma}

\section{Breakdown of the high-dimensional analysis for strong correlations}\label{sec:nonconcentration}

In the subsequent appendices (starting with Appendix~\ref{sec:actual_formula}), we will use random-matrix-theoretic techniques to derive an asymptotically-precise formula for $\mathcal{E}_\text{ICL}(\Gamma^*)$ in the high-dimensional limit. This analysis hinges on our ability to derive what are known as \emph{deterministic equivalents} for the high-dimensional behaviour of random matrices related to the optimal parameter matrix $\Gamma^{*}$ of our reduced linear attention model. Before we embark on the analysis, in this Appendix we derive a condition on the correlations $K$ that is necessary for the random matrix analysis to go through as in \citet{lu2025asymptotictheoryincontextlearning}. This condition is violated when the correlations are strong, which shows that this approach does not obviously extend to the strong correlation regime. 

We recall from \citet{lu2025asymptotictheoryincontextlearning} that we can write the optimal parameter matrix in vectorized form as
\begin{equation}\label{eq:gammastar_definition}\vecop(\Gamma^*) = \qty(\frac{n}{d}\lambda I_{d(d+1)} + \sum_{\mu=1}^n \vecop(H^\mu) \otimes \vecop(H^\mu))^{-1}\sum_{\mu=1}^n y_\text{test}^\mu \vecop(H^\mu).\end{equation}
Starting from this formula, a key ingredient of the high-dimensional analysis is (as we will see in the sequel), the assumption that quadratic forms 
\begin{align}
    \vecop(H^{\mu})^{\top} G \vecop(H^{\mu}),
\end{align}
for $G$ a possibly-random matrix that is independent of $H^{\mu}$, concentrate to their expectations in the high-dimensional limit. To gain some insight into whether this assumption is justified, we consider the test case $G=I$, which corresponds to considering the norms $\Vert \vecop(H^{\mu}) \Vert^2$. If the norms do not concentrate, then for the class of matrices $G$ that appear in the computation, neither should the more general quadratic forms.

To determine when the norms concentrate, suppress the index $\mu \in [n]$, and recall that we have $$H = \q\begin{bmatrix}
    \frac{d}{\ell}({X}\task + \e)^\top {X} & \frac{1}{\ell}({X}\task + \e)^\top({X}\task + \e)
\end{bmatrix}.$$ We will simplify our analysis here, as our goal is to focus on intuition, by ignoring the noise terms given by $\e$; these do not affect the argument. Thus we write $$\vecop(H) = \bm{v} \otimes \q + \text{ignored noise contribution}$$ where $$\bm{v} = \begin{bmatrix}
    \frac{d}{\ell}X^\top X \task \\
    \frac{1}{\ell} \task^\top X^\top X \task 
\end{bmatrix}.$$
We thus have that 
\begin{align*}
    \frac{1}{d^2} \Vert \vecop(H) \Vert^{2} 
    &= \frac{1}{d^2} \Vert \q \Vert^2 \Vert \vv \Vert^{2} 
    \\
    &= \frac{1}{d^2} \Vert \q \Vert^2 \left( \frac{d^2}{\ell^2} \task^{\top} X^{\top} X X^{\top} X \task + \frac{1}{\ell^2} (\task^{\top} X^{\top} X \task)^2 \right)
    \\
    &\overset{d}{=} \left( \frac{1}{d} \vm^{\top} K^{1/2} Z Z^{\top} K^{1/2} \vm + 2 \frac{\sqrt{1-k_{-1}}}{d} \bm{z}^{\top} Z^{\top} K^{1/2} \vm + \frac{1-k_{-1}}{d} \Vert \vm{z} \Vert^2 \right) \nonumber\\&\quad \times \left( \frac{1}{d^2 \ell^2} \task^{\top} Z^{\top} K Z Z^{\top} K Z \task + \frac{1}{d^4 \ell^2} (\task^{\top} Z^{\top} K Z \task)^2 \right), 
\end{align*} where $\overset{d}{=}$ means equivalent in distribution. Here $Z$ is an $\ell \times d$ matrix with i.i.d. $\mathcal{N}(0,1)$ entries and $\bm{z} \sim \mathcal{N}(0,I_d)$ for the query (independent of $Z$).

Upon analysis of each of these terms, which we omit for brevity, we find that the problematic term is 
\begin{align*}
    \frac{1}{d^2 \ell^2} \task^{\top} Z^{\top} K Z Z^{\top} K Z \task
    &\overset{d}{=}\Vert \task \Vert^{2} (Z^{\top} K Z Z^{\top} K Z)_{11}
    \\
    &\overset{d}{=} \frac{\Vert \task \Vert^{2} }{d} \frac{1}{\ell^2} \sum_{a,b=1}^{\ell} \kappa_{a} \kappa_{b} Z_{a 1} Z_{b 1} \frac{1}{d} \sum_{i=1}^{d}  Z_{a i} Z_{b i} \\
    &=     \frac{\Vert \task \Vert^{2} }{d} \frac{1}{d} \left( \frac{1}{\ell} \sum_{a=1}^{\ell} \kappa_{a} Z_{a 1}^{2} \right)^2  + \frac{\Vert \task \Vert^{2} }{d} \frac{1}{\ell^2} \sum_{a,b=1}^{\ell} \kappa_{a} \kappa_{b} Z_{a 1} Z_{b 1} \frac{1}{d} \sum_{i=2}^{d}  Z_{a i} Z_{b i}.
\end{align*} We have arrived at this expression by exploiting the fact that the distribution of the matrix $Z$ is invariant under both left and right rotation to perform two changes of basis: (1) we let $K$ have orthogonal eigendecomposition $K = O\, \text{diag}(\kappa_{1},\ldots,\kappa_{\ell}) \,O^{\top}$, and (2) we choose a basis such that $\task = ( \Vert \task \Vert, 0, \ldots, 0 )^{\top}$. 

Now we must analyse this expression to provide a condition on $K$, specifically its eigenvalues $\kappa_{1},\ldots,\kappa_{\ell}$, for when this term does or does not concentrate. We focus on the second term in the sum, supposing for simplicity that $\Vert \task\Vert^2/d=1$: $$A := \frac{1}{\ell^2} \sum_{a,b=1}^{\ell} \kappa_{a} \kappa_{b} Z_{a 1} Z_{b 1} \frac{1}{d} \sum_{i=2}^{d}  Z_{a i} Z_{b i}.$$ It has moments given by \begin{align*}
    \mathbb{E}[A]
    &= \frac{d-1}{d} \frac{1}{\ell^2} \sum_{a=1}^{\ell} \kappa_{a}^{2}\,,  \\
    \text{var}[A] &= \frac{2}{\ell^4 d^2} (d-1) \left( \sum_{a=1}^{\ell} \kappa_{a}^2 \right)^{2} + \frac{2}{\ell^4 d^2}(d^2-1) \sum_{a=1}^{\ell} \kappa_{a}^{4}\,,
\end{align*} and so the square of the coefficient of variation of $A$ is therefore
\begin{align*}
    \frac{\text{var}[A]}{\mathbb{E}[A]^2} 
    &=  2 \left[\frac{1}{d-1} + \frac{d^2-1}{(d-1)^2} \frac{\sum_{a=1}^{\ell} \kappa_{a}^{4}}{\left( \sum_{a=1}^{\ell} \kappa_{a}^2 \right)^{2}}  \right] \\
    &= 2 [1+\mathcal{O}(d^{-1})] \frac{\sum_{a=1}^{\ell} \kappa_{a}^{4}}{\left( \sum_{a=1}^{\ell} \kappa_{a}^2 \right)^{2}} + \mathcal{O}(d^{-1}).
\end{align*} We now notice that the $K$-dependence here $$\frac{\sum_{a=1}^{\ell} \kappa_{a}^{4}}{\left( \sum_{a=1}^{\ell} \kappa_{a}^2 \right)^{2}} = \frac{\tr(K^4)}{\tr(K^2)^2}$$ looks like some sort of participation ratio. Indeed if $K = I_\ell$ (uncorrelated case, we know this concentrates) then $$\frac{\tr(K^4)}{\tr(K^2)^2} = \frac{1}{\ell} \to 0,$$ while if $K = \bm{1}\bm{1}^\top$ (maximally correlated case, here $\x_1 = \ldots = \x_\ell$) then $$\frac{\tr(K^4)}{\tr(K^2)^2} = 1.$$ In the maximally-correlated case $K = \bm{1}\bm{1}^{\top}$, we see that $A = Z_{1 1}^2 \frac{1}{d} \sum_{i=2}^{d}  Z_{1 i}^2$ tends in distribution to a $\chi^2$ random variable with 1 degree of freedom (as $\frac{1}{d} \sum_{i=2}^{d}  Z_{1 i}^2 \to 1$ almost surely), which clearly does not concentrate. This gives us an intuitive heuristic for concentration of $\|\vecop(H)\|$ :
\begin{tcolorbox}[resultbox, title=Necessary condition for concentration]
Quadratic forms in $\vecop(H)$ converge to a deterministic high-dimensional limit only if
\begin{align}\label{eq:concentration-condition}
    \frac{\tr(K^4)}{\tr(K^2)^2} \to 0.
\end{align}
\end{tcolorbox}
For the exponential case $K_{ab} = \exp(-|a-b|/\xi)$, we can estimate $$\frac{\tr(K^4)}{\tr(K^2)^2} \sim \frac{\xi}{\ell}$$ at large $\ell$, which gives the condition we expected. We see that for correlations that persist on the same order as the context length, we do not have concentration of $\vecop(H)$, and thus we do not have concentration of $\mathcal{E}_\text{ICL}(\Gamma^*)$. An extension of great interest to us would be to analyse this nonconcentration for a wider range of $K$ choices, \textit{e.g.} how does this participation ratio behave for other non-exponential (or even non-Toeplitz) choices of $K$? \\

\noindent \textbf{Remark.} It is worth noting that the case of $k_2 = o(\ell)$ does not violate this condition \eqref{eq:concentration-condition}. Yet, the formula we shall derive will assume that $k_2 = \Theta(1)$. This might potentially be too restrictive of an assumption; indeed Figure \ref{fig:three_regimes} shows great match between simulation and our theory curve, so long as $k_1 = o(\ell)$ as well. We choose to study the $k_0,k_1,k_2 = \Theta(1)$ branch only as this is the branch where the test matrices $A,B,\mathcal{T}$ defined in Lemma \ref{thm:population_icl_risk_CONSTANT} have order 1 entries. Indeed, under numerical investigation, this population loss formula appears to be accurate for the $k_0,k_1,k_2 = o(\ell)$ case. However, we choose not to study this branch in this work due to the elements of $A,B,\mathcal{T}$ becoming extensive in $\ell$.

\section{Asymptotic formula for ICL error in the weak-correlations case}\label{sec:actual_formula}
Here we present a deterministic formula for $\mathcal{E}_\text{ICL}(\Gamma^*)$ in the weak-correlations case. Again, the optimal parameters are given by eq. (\ref{eq:gammastar_definition}). The explicit ridge $\lambda$ is a parameter that adds regularisation to this solution; it is common to take the ridgeless limit $\lambda \to 0$ in similar works. 

Here we have $n$ sample sequences, giving $n$ different $H^\mu \in \mathbb{R}^{d\times(d+1)}$ data embeddings, where the sequence length is of course $\ell + 1$ ($\ell$ main tokens and one query). We will also introduce a new parameter $k$ that controls the number of regression task vectors $\task$ that define the labels in these sequences. The phenomenology surrounding $k$ is not discussed in this work for space reasons, but we include it to best match previous work on this model for consistency \citep{lu2025asymptotictheoryincontextlearning, letey2025pretrain, raventos2023pretraining}. We sample $k$ task vectors $\task_j \iid \mathcal{N}(\bm{0}, I_d)$ and these $k$ vectors are distributed uniformly across the $n$ sequences. This serves to limit the ``task diversity'' of the training batch, allowing us to judge how much of the true task distribution $\mathcal{N}(\bm{0},I_d)$ the model is truly learning in-context. 

We consider the same asymptotic treatment as derived in \cite{lu2025asymptotictheoryincontextlearning}, namely $d \to \infty$ with $$\alpha:=\frac{\ell}{d}\,, \quad \tau:=\frac{n}{d^2}\,, \quad \kappa:= \frac{k}{d}$$ all taken to be $\Theta(1)$ as $d \to \infty$. 

An important quantity to define is the Stieltjes transform of the task sample covariance. This will be given by \begin{equation}
    \M_\kappa(z) := \lim_{d\to \infty, k = \kappa d} \frac{1}{d}\tr\qty(\qty(\frac{1}{k}\sum_{j=1}^k \task_j\task_j^\top + zI_d)^{-1}) 
\end{equation} We will also use $$\M_\kappa'(z):= \frac{\text{d}}{\text{d}z}\M_\kappa(z).$$ These quantities will appear throughout our formula.

The explicit ridge parameter $\lambda$ will be modulated by the fact that we have finite samples, and will appear in our formula through an \textit{effective} ridge defined implicitly by  $$\tilde{\lambda} \M_\kappa\qty(\tilde{\lambda} + \frac{\rho_2}{\alpha}) + \frac{\lambda\tau}{\tilde{\lambda}}= 1-\tau$$ 
We will sometimes write $\M := \M_\kappa(\sigma), \quad \M
':=\M'_\kappa(\sigma)$ for shorthand, where $$\sigma = \tilde{\lambda} + \frac{\rho_2}{\alpha}\,, \quad \tilde{\sigma} = \sigma - \frac{k_0}{\alpha}.$$

Finally, we are only considering the case where the sequential correlations defined by $K,\bm{k}_\text{test}$ are sufficiently weak, \textit{i.e.}, $$k_2:=\frac{1}{\ell}\tr(K^2),\quad k_0 := \bm{k}_\text{test}^\top\bm{k}_\text{test},\quad k_1 := \bm{k}^\top_\text{test}K\bm{k}_\text{test}$$ are all $\Theta(1)$ as $d,\ell\to\infty$. These terms will appear in the constants $$\rho_1 \equiv 1 + \rho\,, \quad \rho_2 \equiv k_2+ \rho\,, \quad \phi_1 \equiv \frac{1}{\alpha^2}(k_1 + \rho k_0)\,, \quad \phi_2 \equiv \frac{1}{\alpha}k_0.$$
We present the main formula below. All proofs are given in subsequent sections.

\begin{proposition}\label{prop:icl-error-formula-constant}
Consider isotropic tokens and tasks, \textit{i.e.}, $\Sigma = I_d = C$, with sequential correlations in the tokens given by PSD $K \in \mathbb{R}^{\ell \times \ell}$ and $\bm{k}_\mathrm{test} \in \mathbb{R}^\ell$ for $\bm{k}_\mathrm{test}^\top K^{-1}\bm{k}_\mathrm{test} \leq 1$. We then have that ICL error for a model defined by $\Gamma^*$ concentrates as $d\to\infty$ in the above limit, \textit{i.e.}, $$\mathcal{E}_\mathrm{ICL}(\Gamma^*) \simeq e_\mathrm{ICL}^\mathrm{corr}(\alpha, \kappa, \tau, \rho, k_0, k_1, k_2)$$ for 
 \begin{align*}
    e_\mathrm{ICL}^\mathrm{corr}(\alpha, \kappa, \tau, \rho, k_0, k_1, k_2) &= e_\mathrm{ICL}^\mathrm{uncorr}\qty(\frac{1+\rho}{k_2+\rho}\alpha, \kappa, \tau, \rho) + e_\mathrm{query}(\alpha, \kappa, \tau, \rho, k_0, k_1, k_2)
\end{align*} where 
\begin{align*}
    e_\mathrm{query} &= -2\phi_2
+ q_{\mathrm{query}}(\M+\tilde{\lambda}\M')
+ \phi_2(2\sigma-\phi_2)\M' + 2\phi_2\tilde{\sigma}\M \\
&\quad+\; \bm{m}_2^\top S^\top M\,S\bm{m}_2
- 2\bm{m}_2^\top S \bm{m}_3
+ 2(1+\phi_2)\bm{m}_1^\top S \bm{m}_2
+ (\phi_1+2\phi_2)\left(1-\tilde{\sigma}\M-\bm{m}_1^\top S \bm{m}_2\right)^2 \\
&\quad
+ \frac{\rho_2}{\alpha}\Bigl[
q_{\mathrm{query}}(\M+\tilde{\lambda}\M')
+ \phi_2(2\sigma-\phi_2)\M'
+ \bm{m}_2^\top S^\top M\,S\bm{m}_2
- 2\bm{m}_2^\top S \bm{m}_3
+ 2\phi_2\M
\Bigr].
\end{align*} for 
$$\quad \bm{m}_1 = \begin{bmatrix}
    \M \\
    \phi_2(1-\sigma \M)
\end{bmatrix}\,,\quad \bm{m}_2 = \begin{bmatrix}
    1-\tilde{\sigma}\M \\
    \phi_2(1 - \tilde{\sigma} + \sigma\tilde{\sigma}\M)
\end{bmatrix}\,, \quad \bm{m}_3 = \begin{bmatrix}
    \M + \tilde{\sigma}\M' \\
    \phi_2(1-\sigma\M - \tilde{\sigma}\M -\sigma\tilde{\sigma}\M')
\end{bmatrix}$$
$$M = \begin{bmatrix}
    -\M' & \phi_2(\M+ \sigma\M')\\
    \phi_2(\M+ \sigma\M') & \phi_2^2(1-2\sigma\M - \sigma^2\M')
\end{bmatrix}\,,\quad S = \begin{bmatrix}
    \M & 1+ \phi_2(1-\sigma \M) \\
    1+ \phi_2(1-\sigma \M) & -\phi_1 + \phi_2^2(1-\sigma + \sigma^2\M)
\end{bmatrix}^{-1}$$
\begin{align*} q_\mathrm{query} &= \frac{2\frac{k_0}{\alpha}\qty(\tilde{\lambda}(\M + \sigma\M')+(1-2\sigma\M)) + \frac{k_0^2}{\alpha^2}\qty(\M-\tilde{\lambda}\M') + \bm{m}_2^\top S\bm{m}_2 - \tilde{\lambda}\qty(- 2\bm{m}_2^\top S\bm{m}_3 + \bm{m}_2^\top S^\top MS\bm{m}_2)}{\tau - (1-2\tilde{\lambda}\M - \tilde{\lambda}^2\M')}.\end{align*}
\end{proposition}

\section{Proofs of Population Averages}
We begin for necessary $y_\text{test}, H$ averages needed for terms in the population loss formula in Lemma \ref{thm:population_icl_risk_CONSTANT}. \textit{A language model was used at various stages of these proofs to help with the Wick expansions of the 6th order moments in terms of traces. All expressions were subsequently carefully verified by the authors both analytically and numerically.}

\begin{lemma}\label{lemma:general_k_token_expectations_FINITE}
Suppose we have $k_2, k_0, k_1 = \Theta(1)$ \textit{i.e.}, weak correlations. Write $\bm{v} = \Sigma\task.$ Then we have \begin{align*}
        \E[y_\mathrm{test}^2] &= \tr[\Sigma \task\task^\top] + \rho \\
        \E[y_\mathrm{test} H] &\approx \frac{1}{d}\begin{bmatrix}
        \bm{v}\bm{v}^\top + \frac{1}{\alpha}k_0\tr[\Sigma\task\task^\top]\Sigma & \qty(\tr[\Sigma \task\task^\top] + \rho)\bm{v}
    \end{bmatrix} \\
        \E[\mathrm{vec}(H)\mathrm{vec}(H)^\top] &\approx \frac{1}{d} \Sigma \otimes \begin{bmatrix}
        \bm{v}\bm{v}^\top + \qty(k_2\tr[\Sigma \task\task^\top] + \rho)\Sigma/\alpha & \qty(\tr[\Sigma \task\task^\top] + \rho)\bm{v} \\
        \qty(\tr[\Sigma \task\task^\top] + \rho)\bm{v}^\top & \qty(\tr[\Sigma \task\task^\top] + \rho))^2 
    \end{bmatrix} + \frac{1}{d^2}\mathcal{X}
    \end{align*}
for $d(d+1)\times d(d+1)$ tensor $\mathcal{X}$ given by \begin{align*}
\mathcal{X}
&=
\frac{1}{\alpha^2}\Bigl(k_1 \operatorname{tr}[\Sigma W] + \rho k_0\Bigr)\,
\operatorname{vec}\qty(\bigl[\Sigma \;\; \bm{0}\bigr])\operatorname{vec}\qty(\bigl[\Sigma \;\; \bm{0}\bigr])^{\top} \\
&\quad+\; \frac{k_0}{\alpha}
\Bigl(
\operatorname{vec}\qty(\bigl[\Sigma \;\; \bm{0}\bigr])\operatorname{vec}\qty(\bigl[S \;\;  \qty(\tr[\Sigma \task\task^\top] + \rho)\bm{v}\bigr])^{\top}
\\&\quad+
\operatorname{vec}\qty(\bigl[S \;\;  \qty(\tr[\Sigma \task\task^\top] + \rho)\bm{v}\bigr])\operatorname{vec}\qty(\bigl[\Sigma \;\; \bm{0}\bigr])^{\top}
\Bigr).
\end{align*}
The ``$\approx$'' means subleading terms in $\ell,d$ are neglected. 
\end{lemma}

\begin{proof} Let's go term by term. 

\noindent \textcolor{Periwinkle}{Label-label term.} Conditioning on $X$, we have
\[
\mathbb{E}[\q\q^\top\mid X]= \frac{1}{d}(1-k_{-1})\Sigma+(X^\top\bm{m})(X^\top\bm{m})^\top
\] and so \begin{align*}
    \E[y_\mathrm{test}^2] &= \E[\tr(\q\q^\top \task\task^\top) + \rho] \\
    &= \frac{1}{d}(1-k_{-1})\E_X\qty[\tr(\Sigma \task\task^\top)] + \E_X\qty[\tr((X^\top\bm{m})(X^\top \bm{m})^\top\task\task^\top)] + \rho \\
    &= (1-k_{-1}+k_{-1})\tr[\Sigma \task\task^\top] + \rho.
\end{align*}

\noindent \textcolor{Periwinkle}{Label-feature term.} 
For \begin{align*}
    \bm{b} = X^\top(X\task+ \e)\,, \quad c = (X\task+ \e)^\top (X\task+ \e) 
\end{align*} we have $$y_\mathrm{test} H = \q(\q^\top \task + \varepsilon_\mathrm{test})\begin{bmatrix}
    \frac{d}{\ell}\bm{b}^\top & \frac{1}{\ell} c
\end{bmatrix}$$ and so \begin{align*}
    \E[y_\mathrm{test}H] &= \frac{1}{d}(1-k_{-1})\Sigma \task\, \begin{bmatrix}
    \frac{d}{\ell}\E_X\qty[\bm{b}] \\
    \frac{1}{\ell} \E_X\qty[c]
\end{bmatrix}^\top + \begin{bmatrix}
    \frac{d}{\ell}\E_X[(X^\top \bm{m})(X^\top \bm{m})^\top \task\, \bm{b}^\top ] &
    \frac{1}{\ell}\E_X[(X^\top \bm{m})(X^\top \bm{m})^\top \task\, c]
\end{bmatrix}
\end{align*}
Have \begin{align*}
    \frac{d}{\ell}\E_X\qty[\bm{b}] &= \tr[K]\bm{v} \\
    \frac{1}{\ell}\E_X[c] &= \tr[K]\tr[\Sigma\task\task^\top] + \rho \\
    \frac{d}{\ell}\E_X[(X^\top \bm{m})(X^\top \bm{m})^\top \task \bm{b}^\top] &= \frac{1}{d}\cdot \frac{1}{\ell}\qty(k_{-1}\tr(K) + k_0)\bm{v}\bm{v}^\top + \frac{1}{d}\cdot\frac{1}{\ell}k_0\tr(\Sigma \task\task^\top) \Sigma \\
    &\approx \frac{1}{d}\cdot\qty(k_{-1}\tr[K]\bm{v}\bm{v}^\top + \frac{1}{\alpha}k_0\tr[\Sigma \task\task^\top]\Sigma) \\
    \frac{1}{\ell}\E_X[(X^\top \bm{m})(X^\top \bm{m})^\top \task\, c] &= \frac{1}{d}\cdot \rho k_{-1}\bm{v} + \frac{1}{d}\cdot k_{-1}\tr[K]\tr[\Sigma \task \task^\top]\bm{v} + \frac{2}{d^2}\frac{k_0}{\ell}\bm{v}\bm{v}^\top \task  \\
    &\approx \frac{1}{d}\cdot k_{-1}\qty(\tr[K]\tr[\Sigma \task\task^\top] + \rho)\bm{v}
    \end{align*} 
Combining gives \begin{align*}
    \E[y_\mathrm{test}H] &= \frac{1}{d}\begin{bmatrix}
        \tr[K]\bm{v}\bm{v}^\top + \frac{1}{\alpha}k_0\tr[\Sigma\task\task^\top]\Sigma & \rho_1\bm{v}
    \end{bmatrix}
\end{align*}

\noindent \textcolor{Periwinkle}{Feature-feature term}
Using $\bm{b},c$ as above we can write more easily
\begin{align*}
    \text{Vec}(H)\text{Vec}(H)^\top &= (\q\q^\top) \otimes \qty(\begin{bmatrix}
    \frac{d}{\ell}\bm{b}\\ \frac{1}{\ell}c
\end{bmatrix} \begin{bmatrix}
    \frac{d}{\ell}\bm{b}\\ \frac{1}{\ell}c
\end{bmatrix}^\top)
\end{align*} 
Taking conditional expectation over $\q$ for fixed $X$, our expression simplifies as 
\[
 {
\mathbb{E}\!\left[\mathrm{Vec}(H)\mathrm{Vec}(H)^\top\right]
= \frac{1}{d}(1-k_{-1})\Sigma \otimes \mathbb{E}_{X,\epsilon}\!\left[\bm{u}\bm{u}^\top\right] + \mathbb{E}_{X,\epsilon}\!\left[(X^\top
\bm{m})(X^\top\bm{m})^\top\otimes(\bm{u}\bm{u}^\top)\right].
}
\] where I'm writing $\bm{u}^\top = \begin{bmatrix}
    \frac{d}{\ell}\bm{b}^\top & \frac{1}{\ell}c
\end{bmatrix}$ for convenience. Now let's start with the $\E[\bm{u}\bm{u}^\top]$ term. 

We have \begin{align*}
    \E[\bm{b}\bm{b}^\top] &= \E[X^\top X \task\task^\top X^\top X] + \E[X^\top \e\e^\top X] \\
    &= \frac{1}{d^2}\qty(\tr(K^2)+\tr(K)^2){\Sigma}\task\task^\top{\Sigma} + \frac{1}{d^2}\tr(K^2)\tr(\Sigma \task\task^\top){\Sigma} + \frac{1}{d}\rho \tr(K)\Sigma\\
    \E[c\bm{b}] &= \E[\task^\top X^\top X \task X^\top X \task] + 2\E[\e^\top X w X^\top \e] + \E[\e^\top \e X^\top X \task] \\
    &= \qty(\frac{1}{d^2}\qty(\tr(K)^2 + 2\tr(K^2))\tr(\Sigma\task\task^\top)+\frac{1}{d}\rho(\ell+2)\tr(K))\Sigma\task \\
    \E[c^2] &= \E[\task^\top X^\top X \task \task^\top X^\top X \task] + 2 \E[\task^\top X^\top X \task \e^\top \e] + 4\E[\task^\top X^\top \e \e^\top X \task] + \E[\e^\top\e\e^\top \e] \\
    &= \frac{1}{d^2}\qty(\tr(K)^2 + 2\tr(K^2))\tr(\Sigma\task\task^\top)^2 + \frac{1}{d}\rho(2\ell+4)\tr(K)\tr(\Sigma\task\task^\top) + \rho^2(\ell^2+2\ell).
\end{align*} 
This simplifies as \begin{align*}
    \frac{d^2}{\ell^2}\E[\bm{b}\bm{b}^\top] &\approx \tr[K]^2 (\Sigma\task)(\Sigma \task)^\top + \frac{1}{\alpha}\qty(\rho\tr[K]+\tr[K^2]\tr[\Sigma W])\Sigma \\
    \frac{d}{\ell^2}\E[c\bm{b}] &\approx \tr[K]\qty(\tr[K]\tr[\Sigma W] + \rho) \Sigma \task \\
    \frac{1}{\ell^2}\E[c^2] &\approx \qty(\rho + \tr[K]\tr[\Sigma W])^2
\end{align*} upon making the above high-dimensional assumptions.  

For the higher order term, $$\mathbb{E}\left[(X^\top
\bm{m})(X^\top\bm{m})^\top\otimes(\bm{u}\bm{u}^\top)\right] = \begin{bmatrix}
    \frac{d^2}{\ell^2}\E[(X^\top
\bm{m})(X^\top\bm{m})^\top\otimes\bm{b}\bm{b}^T] & \frac{d}{\ell^2}\E[(X^\top
\bm{m})(X^\top\bm{m})^\top\otimes c\bm{b}] \\
\frac{d}{\ell^2}\E[(X^\top
\bm{m})(X^\top\bm{m})^\top\otimes c\bm{b}^T] & \frac{1}{\ell^2}\E[(X^\top
\bm{m})(X^\top\bm{m})^\top\otimes c^2]
\end{bmatrix}$$
Write 
\[
\bm{v} := \Sigma \task,\qquad S:=\bm{v}\bm{v}^\top\,, \qquad 
k_{-1}\equiv \mathrm{tr}(KM) = \bm{k}_\mathrm{test}^\top K^{-1}\bm{k}_\mathrm{test},\quad
k_0 \equiv \bm{k}_\mathrm{test}^\top \bm{k}_\mathrm{test} ,\quad
k_1 \equiv  \bm{k}_\mathrm{test}^\top K\bm{k}_\mathrm{test}
\]

For large $\ell, d$ and the chosen high-dimensional assumptions, the Isserlis expansion gives \begin{align*}
\frac{d^2}{\ell^2} \, \mathbb{E}\Big[\big((X^\top m)(X^\top m)^\top\big)\otimes (\bm{b}\bm{b}^\top)\Big]_{(I,i),(J,j)}
&\approx \frac{1}{d}k_{-1}\qty( \Sigma_{IJ}S_{ij}\tr[K]^2 + \frac{1}{\alpha}
\qty(\tr[\Sigma W]\tr[K^2]+ \rho \tr[K])\Sigma_{IJ}\Sigma_{ij}) \\
&\quad+\; \frac{1}{d^2}\mathcal{X}_{(I,i)(J,j)}
\end{align*} with \begin{align*}
\mathcal{X}_{(I,i)(J,j)} &\equiv \frac{1}{\alpha^2}k_1\tr[\Sigma W] \Sigma_{Ii}\Sigma_{Jj} + \frac{1}{\alpha}\tr[K]k_0\qty(\Sigma_{Ii}S_{Jj}+\Sigma_{Jj}S_{Ii}) + \rho\frac{1}{\alpha^2}k_0\Sigma_{Ii}\Sigma_{Jj} \\
&= \frac{1}{\alpha^2}\qty(k_1\tr[\Sigma W] + \rho k_0)\Sigma_{Ii}\Sigma_{Jj} + \frac{1}{\alpha}k_0\tr[K]\qty(\Sigma_{Ii}S_{Jj}+\Sigma_{Jj}S_{Ii}).
\end{align*}
Now for the second term, \begin{align*}
    \E[(X^\top
\bm{m})(X^\top\bm{m})^\top\otimes c \bm{b}]_{(I,i),J} &= \E[c(X^\top \bm{m})_I(X^\top \bm{m})_J b_i] \\
&= \E[(X_{sk}w_k + \varepsilon_s)(X_{sn}w_n + \varepsilon_s) X_{tI}X_{pJ}M_{tp} X_{qi}(X_{qj}w_j + \varepsilon_q)] \\
&= M_{tp}W_{kn}w_j \E[X_{sk}X_{sn}X_{tI}X_{pJ}X_{qi}X_{qj}] + (\ell+2)\rho M_{tp}w_j \E[X_{tI}X_{pJ}X_{qi}X_{qj}] 
\end{align*}
Expanding all the Wick terms gives 
\[
{
\begin{aligned}
\frac{d}{\ell^2}\mathbb{E}\Big[\big((X^\top m)(X^\top m)^\top\big)\otimes (cb)\Big]_{(I,i),J}
&\approx \frac{1}{d}k_{-1}\qty(\tr[\Sigma W] + \rho )\Sigma_{IJ}v_i + \frac{1}{d^2}\mathcal{X}_{(I,i)(J,d+1)}
\end{aligned}}
\]
where $$\mathcal{X}_{(I,i)(J,d+1)} \equiv \frac{1}{\alpha}k_0\qty(\tr[\Sigma W] +\rho)\Sigma_{Ii}v_{J}.$$
Finally, the last term gives 
\[
{
\begin{aligned}
\frac{1}{\ell^2}\mathbb{E}\Big[((X^\top m)(X^\top m))\otimes c^2\Big]_{IJ}
&\approx \frac{1}{d}k_{-1}\qty(\tr[\Sigma W] + \rho)^2\Sigma_{IJ} 
\end{aligned}}
\]
Gathering everything together gives the required formula.
\end{proof}

\newpage

\section{Proof of Proposition \ref{prop:icl-error-formula-constant}}
This entire section is an extended calculation that serves as a complete proof of the deterministic formula for $\mathcal{E}_\text{ICL}(\Gamma^*)$ given in Proposition \ref{prop:icl-error-formula-constant}. The methodology in this section is to analyse the random object $\Gamma^*$, which contains randomness through the particular training batch $(H^\mu, y_\text{test}^\mu)_{\mu=1}^n$, as a resolvent of a larger random matrix, and the necessary $\tr(\Gamma A^\top)$ and $\tr(\Gamma^\top B \Gamma)$ terms as traces of a resolvent against deterministic test matrices. Throughout this entire section we will take $\Sigma = I_d$. 

This section will be relatively light on detail as the full calculation setup can be found in both \cite{lu2025asymptotictheoryincontextlearning} and \cite{letey2025pretrain}. For a more rigorous analysis, with error terms bounded properly, see \cite{lu2025asymptotictheoryincontextlearning}; here we will simply write $\approx$ when terms can be dropped due to the negligibility in high dimensions. 

We care about the MSE loss (computed by Theorem \ref{thm:population_icl_risk_CONSTANT}) for the optimal parameters 
\begin{align}\label{eq:gamma_star_actual_formula_appendix}
\text{vec}(\Gamma^\ast) =
\left(\frac{n}{d}\lambda I + \sum_{\mu=1}^{n}\text{vec}(H^\mu)\text{vec}(H^\mu)^\top\right)^{-1} \sum_{\mu=1}^{n}y_\mathrm{test}^{\mu}\text{vec}(H^\mu)\,
\end{align} for $\mu \in [n]$ denoting the index for the $n$ training sample sequences. Define $$ \bm{z}_\mu = \begin{bmatrix}y_\mathrm{test}^\mu / d \\ \vecop(H^\mu)/\sqrt d\end{bmatrix} \in \mathbb{R}^{d(d+1)+1}$$ and construct extended resolvent 
\begin{equation}\label{eq:Ge}
    G_\mathrm{ext}(\pi) = \frac{1}{\sum_{\mu \in [n]} \bm{z}_\mu \bm{z}_\mu^\top + \pi B_\mathrm{ext} + \tau\lambda I}\,.
\end{equation} for some test matrix $B_\mathrm{ext} \in \mathbb{R}^{(d(d+1)+1)\times(d(d+1)+1)}$. 

It's important to note that even though we're studying sequentially-correlated data, the vectors $\bm{z}_\mu$s are still independent of each other. We can thus apply the ``leave-one-out'' or cavity method here over the $\mu$ index. We end up with \begin{equation}\label{eq:Ge_loo_id}
    \sum_{\mu \in [n]} \frac{1}{1 + \bm{z}_\mu^\top G_\mathrm{ext}^{[\mu]} \bm{z}_\mu} G_\mathrm{ext}^{[\mu]} \bm{z}_\mu \bm{z}_\mu^\top + G_\mathrm{ext}  (\pi B_\mathrm{ext} + \tau\lambda I)= I.
\end{equation}
For fixed task $\task_\mu$ we can apply Lemma \ref{lemma:general_k_token_expectations_FINITE} to compute $$\E_{X,\q,\e}[\bm{z}\bm{z}^\top] = \frac{1}{d^2}\Upsilon(\task)$$ for 
\begin{align*}
\Upsilon(\task) &\equiv \begin{bmatrix}
        \tr[\task\task^\top] + \rho & \frac{1}{\sqrt{d}}\vecop\qty(\begin{bmatrix} \task\task^\top + \frac{1}{\alpha}k_0\tr[\task\task^\top]I_d& \rho_1(\task) \task\end{bmatrix}) \\
        \frac{1}{\sqrt{d}}\vecop\qty(\begin{bmatrix} \task\task^\top + \frac{1}{\alpha}k_0\tr[\task\task^\top]I_d & \rho_1(\task) \task\end{bmatrix})^\top &I_d \otimes E(\task) + \frac{1}{d}\mathcal{X}(\task)
    \end{bmatrix}
\end{align*}
where \begin{align*}
    E(\task) \equiv \begin{bmatrix}
            \task\task^\top + \rho_2(\task) I_d/\alpha &  \rho_1(\task) \task \\
             \rho_1(\task) \task^\top & \rho_1(\task)^2
        \end{bmatrix}\,, \quad \rho_1(\task)= \tr[ \task\task^\top] + \rho, \quad \rho_2(\task) = \tr[\task\task^\top]k_2 + \rho
\end{align*} 
and $d(d+1)\times d(d+1)$ tensor 
\begin{align*}
    \mathcal{X}(\task) &\equiv \frac{1}{\alpha^2}\qty(\tr[\task\task^\top]k_1 + \rho k_0)\vecop\qty(\begin{bmatrix}
        I_d & \bm{0}
    \end{bmatrix})\vecop\qty(\begin{bmatrix}
        I_d & \bm{0}
    \end{bmatrix})^\top \\
    &\quad+\; \frac{1}{\alpha} k_0\qty(\vecop\qty(\begin{bmatrix}
        I_d & \bm{0}
    \end{bmatrix})\vecop\qty(\begin{bmatrix}
        \task\task^\top & \bm{0}
    \end{bmatrix})^\top + \vecop\qty(\begin{bmatrix}
        \task\task^\top & \bm{0}
    \end{bmatrix})\vecop\qty(\begin{bmatrix}
        I_d & \bm{0}
    \end{bmatrix})^\top) \\
    &\quad+\; \frac{1}{\alpha}\rho_1(\task) k_0\qty(\vecop\qty(\begin{bmatrix}
        I_d & \bm{0}
    \end{bmatrix})\vecop\qty(\task\bm{e}_{d+1}^\top)^\top + \vecop\qty(\task\bm{e}_{d+1}^\top)\vecop\qty(\begin{bmatrix}
        I_d & \bm{0}
    \end{bmatrix})^\top).
\end{align*}
The quadratic form $\bm{z}_\mu^\top G_\mathrm{ext}^{[\mu]} \bm{z}_\mu$ concentrates for fixed tasks $\task_\mu$ (assuming weak correlations, so that the condition noted in Appendix~\ref{sec:nonconcentration} is not violated), so we have 
\begin{equation}
    \bm{z}_\mu^\top G_\mathrm{ext}^{[\mu]} \bm{z}_\mu \simeq \chi^\mu(\bm{w}_\mu) 
\end{equation}
where
\begin{equation}\label{eq:zmu_quadratic}
    \chi^\mu(\bm{w}_\mu) \equiv \frac{1}{d^2} \tr\left(\left[G_\mathrm{ext}^\mu\right]_{\setminus 0} \cdot \left[I \otimes E(\bm{w}_\mu) + \frac{1}{d}\mathcal{X}(\task_\mu)\right]\right).
\end{equation}
Replacing $\bm{z}_\mu^\top G_\mathrm{ext}^{[\mu]} \bm{z}_\mu$ in (\ref{eq:Ge_loo_id}) with $\chi^\mu(\bm{w}_\mu)$ gives
\begin{equation}\label{eq:Ge_loo_id_quadratic}
    \sum_{\mu \in [n]} \frac{1}{1 + \chi^\mu(\bm{w}_\mu)} G_\mathrm{ext}^{[\mu]} \bm{z}_\mu \bm{z}_\mu^\top + G_\mathrm{ext}  (\pi B_\mathrm{ext} + \tau\lambda I) \simeq I\,.
\end{equation}
In this equation we will also replace $\bm{z}_\mu\bm{z}_\mu^\top$ with its conditional expectation over $X,\q,\e$, giving 
\begin{align}\label{eq:Ge_loo_id_expectation}
    \frac{\tau}{n}\sum_{\mu \in [n]} \frac{1}{1 + \chi^\mu(\bm{w}_\mu)} G_\mathrm{ext}^{[\mu]} \Upsilon(\task_\mu) + G_\mathrm{ext}  (\pi B_\mathrm{ext} + \tau\lambda I) \simeq I 
\end{align}
where recall $\tau = n/d^2$. 

In high dimensions and for large $n$, there is negligible difference between $\sum_{\nu\neq\mu}$ and $\sum_\mu$. Thus, we replace $G_\mathrm{ext}^\mu$ by $G_\mathrm{ext}$, and $\chi^\mu(\bm{w}_\mu)$ by
\begin{equation}\label{eq:zmu_quadratic_full}
    \chi(\bm{w}_\mu) \equiv \frac{1}{d^2} \tr\left(\left[G_\mathrm{ext}\right]_{\setminus 0} \cdot \left[I \otimes E(\bm{w}_\mu) + \frac{1}{d}\mathcal{X}(\task_\mu)\right]\right).
\end{equation}
So finally we have the expression for $G_\mathrm{ext}$
\begin{align}
    G_\mathrm{ext}\left(\frac{\tau}{n}\sum_{\mu \in [n]} \frac{1}{1 + \chi(\bm{w}_\mu)} \Upsilon(\task_\mu) + \pi B_\mathrm{ext} + \tau\lambda I\right) \simeq I \label{eq:Ge_equiv_0} \,.
\end{align}

\noindent \textcolor{Periwinkle}{\textit{Exploit finiteness of training task set.}} So far we are summing over $n$ task vectors, but really only $n/k$ of these are unique. Thus, we can simplify (\ref{eq:Ge_equiv_0}) as 
\begin{equation}\label{eq:Ge_equiv_1}
    G_\mathrm{ext}\left(\frac{\tau}{k}\sum_{j \in [k]} \frac{1}{1 + \chi(\task_j)} \Upsilon(\task_j) + \pi B_\mathrm{ext} + \tau\lambda I\right) \simeq I\,.
\end{equation}
We replace $\chi(\task_j)$, which is self-averaging in $\task_j$, with its mean \begin{equation}\label{eq:chi_n}
    \widehat\chi_\mathrm{ave} \equiv \frac{1}{k} \sum_{j \in [k]} \chi(\task_j).
\end{equation} To clean up the sums over the tasks $\task_1,...\task_k$ we use the fact that \begin{align*}
    \frac{1}{k}\sum_{j\in[k]} \tr[\task_j \task_j^\top] &=  \tr[\Rtr] \\
    \frac{1}{k} \sum_{j\in[k]} \left(\rho+ \tr[\task_j \task_j^\top]\right)\task_j &\simeq\left(\rho +  \tr[\Rtr]\right)\btr \\
    \frac{1}{k}\sum_{j\in[k]} \left(\rho +  \tr[\task_j \task_j^\top]\right)^2 &\simeq\left(\rho + \tr[\Rtr]\right)^2 
\end{align*} for $$
    \btr \equiv \frac{1}{k}\sum_{j\in[k]}\task_j\,, \qquad 
    \Rtr \equiv \frac{1}{k}\sum_{j\in[k]}\task_j\task_j^\top. $$
We thus have that \begin{align*}
    \frac{1}{k}\sum_{j\in[k]} E(\task_j) &\approx B_\mathrm{tr} \equiv \begin{bmatrix}
        \Rtr + \rho_2I_d/\alpha & \rho_1 \btr \\
        \rho_1 \btr^\top & \rho_1^2 
    \end{bmatrix} \\
    \frac{1}{k}\sum_{j\in[k]}\mathcal{X}(\task_j) &\approx \Phi
\end{align*}
for \begin{align*}
    \Phi &\equiv \frac{1}{\alpha^2}\qty(\tr[\Rtr]k_1 + \rho k_0)\vecop\qty(\begin{bmatrix}
        I_d & \bm{0}
    \end{bmatrix})\vecop\qty(\begin{bmatrix}
        I_d & \bm{0}
    \end{bmatrix})^\top \\
    &\quad+\; \frac{1}{\alpha} k_0\qty(\vecop\qty(\begin{bmatrix}
        I_d & \bm{0}
    \end{bmatrix})\vecop\qty(\begin{bmatrix}
        \Rtr & \bm{0}
    \end{bmatrix})^\top + \vecop\qty(\begin{bmatrix}
        \Rtr & \bm{0}
    \end{bmatrix})\vecop\qty(\begin{bmatrix}
        I_d & \bm{0}
    \end{bmatrix})^\top) \\
    &\quad+\; \frac{1}{\alpha}\rho_1 k_0\qty(\vecop\qty(\begin{bmatrix}
        I_d & \bm{0}
    \end{bmatrix})\vecop\qty(\btr\bm{e}_{d+1}^\top)^\top + \vecop\qty(\btr\bm{e}_{d+1}^\top)\vecop\qty(\begin{bmatrix}
        I_d & \bm{0}
    \end{bmatrix})^\top)
\end{align*}
where $$\rho_1 = \tr[\Rtr]+\rho\,,\qquad \rho_2 = \tr[K^2]\tr[\Rtr] + \rho.$$
Finally, we have 
\begin{equation}\label{eq:chi_ave}
    \widehat\chi_\mathrm{ave} = \frac{1}{d^2} \tr\left(\left[G_\mathrm{ext}\right]_{\setminus 0} \cdot \left[I \otimes B_\mathrm{tr} + \frac{1}{d}\Phi\right]\right).
\end{equation} 

Thus, after averaging over $X,\q,\e$ in the extended resolvent $G_\mathrm{ext}$, we have a deterministic equivalent $G_\mathrm{ext} \simeq \G_\mathrm{ext}$ (still depending on random task quantities $\Rtr,\btr$) defined by self-consistent equations
\begin{align}
    \left[\G_\mathrm{ext}\right]_{\setminus 0} &= \qty(\frac{\tau}{1+\chi_\pi}I_d\otimes B_\mathrm{tr} + \frac{1}{d}\frac{\tau}{1+\chi_\pi}\Phi + \pi \Pi + \tau \lambda I_d \otimes I_{d+1})^{-1} \label{eq:full_tensor_Gext0}\\
    \chi_\pi &= \frac{1}{d^2}\tr\qty(\left[\G_\mathrm{ext}\right]_{\setminus 0} \qty(I_d \otimes B_\mathrm{tr} + \frac{1}{d}\Phi)) \label{eq:full_tensor_chipi}
\end{align}
where the full matrix is given by \begin{align}\label{eq:G_equiv_hat}
    \G_\mathrm{ext}(\pi)^{-1} &\equiv \frac{\tau}{1+\chi_\pi}  \begin{bmatrix}
    \rho_1 & \frac{1}{\sqrt d} \vecop\left(\begin{bmatrix}\Rtr + \frac{1}{\alpha}k_0\tr[\Rtr]I_d &  \rho_1\btr\end{bmatrix}\right)^\top\\
    \frac{1}{\sqrt d} \vecop\left(\begin{bmatrix}\Rtr + \frac{1}{\alpha}k_0\tr[\Rtr]I_d &  \rho_1\btr\end{bmatrix}\right) & I_d \otimes B_\mathrm{tr} + \frac{1}{d}\Phi
    \end{bmatrix} \nonumber\\
    &\quad+\; \pi B_\mathrm{ext} + \tau\lambda I
\end{align}

\textit{Note that up until this point, the results have matched the previous analysis from \cite{lu2025asymptotictheoryincontextlearning} and \cite{letey2025pretrain} with the exception of $\rho_2$ (depending on $k_2$) and the low-rank term $\Phi$ (depending on $k_0$ and $k_1$).}

The low-rank term is given by 
\begin{align*}
    \Phi &= \phi_1 1_+ 1_+^\top + \phi_2(1_+ r^\top + r 1_+^\top) + \phi_3(1_+  \mu^\top + \mu 1_+^\top) \\
    &= \begin{bmatrix}
        1_+ & v
    \end{bmatrix}\begin{bmatrix}
        \phi_1 & 1 \\
        1 & 0
    \end{bmatrix}\begin{bmatrix}
        1_+^\top  \\
        v^\top
    \end{bmatrix}  \\
    &= T \phi T^\top 
\end{align*} 
for $$1_+ \equiv \vecop\qty(\begin{bmatrix} I_d & \bm{0}\end{bmatrix})\,, \qquad r \equiv \vecop\qty(\begin{bmatrix} \Rtr & \bm{0}\end{bmatrix})\,, \qquad \mu \equiv \vecop\qty( \btr \bm{e}_{d+1}^\top)\,, \qquad v \equiv \vecop\qty(\begin{bmatrix}
    \phi_2\Rtr & \phi_3\btr
\end{bmatrix})$$
and $$\phi_1 \equiv \frac{1}{\alpha^2}(\tr[\Rtr]k_1 + \rho k_0)\,, \qquad \phi_2 \equiv \frac{1}{\alpha}k_0\,, \qquad \phi_3 \equiv \frac{1}{\alpha}(\tr[\Rtr] +\rho)k_0.$$

\subsection{Effective ridge}
The effective ridge will be $\tilde{\lambda} = \lambda(1+\chi_0)$ where $\chi_0$ is defined as the solution to the implicit equations
\begin{align}
\left[\G_\mathrm{ext}\right]_{\setminus 0} &= \qty(\frac{\tau}{1+\chi_0}I_d\otimes B_\mathrm{tr} + \frac{1}{d}\frac{\tau}{1+\chi_0}\Phi + \tau \lambda I_d \otimes I_{d+1})^{-1} \\
    \chi_0 &= \frac{1}{d^2}\tr\qty(\left[\G_\mathrm{ext}\right]_{\setminus 0} \qty(I_d \otimes B_\mathrm{tr} + \frac{1}{d}\Phi)).
\end{align}
We can recast the above implicit equations using Woodbury, as  \begin{align}\label{eq:selfconslambda}
    \frac{\tau \chi_0}{1+\chi_0} = \frac{1}{d^2}\tr\qty(\qty(I_d \otimes F_0 - \frac{1}{d}\Phi_C)\qty(I_d \otimes B_\mathrm{tr} + \frac{1}{d}\Phi))
\end{align} where \begin{align*}
    F_0 &= \qty(B_\mathrm{tr}+ \lambda(1+\chi_0)I_{d+1})^{-1} \\
    \Phi_C &= (I_d \otimes F_0)\begin{bmatrix}
        1_+ & v
    \end{bmatrix}S\begin{bmatrix}
        1_+^\top \\
        v^\top
    \end{bmatrix}(I_d \otimes F_0) \qquad \text{for some $2\times 2$ matrix $S$ with $\Theta(1)$ elements.}
\end{align*}
However, all the $1/d$ terms in \eqref{eq:selfconslambda} are subleading, e.g. \begin{align*}
    \frac{1}{d^3}\tr\qty((I_d\otimes F_0)rr^\top(I_d \otimes F_0B_\mathrm{tr})) &= \frac{1}{d^3}\tr(\Rtr F_0 B_\mathrm{tr} F_0 \Rtr) = \mathcal{O}\qty(\frac{1}{d^2})
\end{align*} and so we write 
\begin{align}\label{eq:simplified_lambdatilde} \frac{\tau \chi_0}{1+\chi_0}\approx \frac{1}{d^2}\tr\qty(I_d\otimes F_0B_\mathrm{tr}) = \tr[(B_\mathrm{tr} + \tilde{\lambda}I_{d+1})^{-1}B_\mathrm{tr}].\end{align}
This can be simplified as 
\begin{align*}
    \frac{\tau\chi_0}{1+\chi_0} &= 1-\lambda(1+\chi_0)\tr[\qty(B_\mathrm{tr} + \lambda(1+\chi_0) I_{d+1})^{-1}] \\
    &\approx 1-\lambda(1+\chi_0)\tr[\qty(\Rtr + \qty(\lambda(1+\chi_0)+\frac{\rho_2}{\alpha}) I_{d})^{-1}]
\end{align*} after ignoring the $\btr$ terms in $B_\mathrm{tr}$. This is where the Stieltjes transform of $\Rtr$ is first introduced: $$\M_\kappa(w) = \lim_{d\to \infty, k\to\infty, k/d = \kappa} \frac{1}{d}\tr\qty(\qty(\Rtr + w I_d)^{-1}).$$ Using this, we find the effective ridge self-consistency equation 
\begin{equation}\label{eq:tildelambda}
    \tilde{\lambda} \M_\kappa\qty(\tilde{\lambda} + \frac{\rho_2}{\alpha}) + \frac{\lambda\tau}{\tilde{\lambda}}= 1-\tau
\end{equation}
\subsection{Relating parameters to resolvent}
Recall that $$
\text{vec}(\Gamma^\ast) =
\left(\frac{n}{d}\lambda I + \sum_{\mu=1}^{n}\text{vec}(H^\mu)\text{vec}(H^\mu)^\top\right)^{-1} \sum_{\mu=1}^{n}y_\mathrm{test}^{\mu}\text{vec}(H^\mu)$$ and so we have 

\begin{align*}
    G_\mathrm{ext}(0) = d \begin{bmatrix}
        \frac{1}{d}\sum_\mu (y^\mu_\mathrm{test})^2 + \frac{n}{d}\lambda  & \frac{1}{\sqrt{d}}\sum_\mu y^\mu_\mathrm{test}\vecop(H^\mu)^\top \\
        \frac{1}{\sqrt{d}}\sum_\mu y^\mu_\mathrm{test}\vecop(H^\mu) & \sum_\mu \vecop(H^\mu)\vecop(H^\mu)^\top + \frac{n}{d}\lambda I
    \end{bmatrix}^{-1}
\end{align*}
Using that $$\begin{bmatrix}
    a & \bm{b}^\top \\
    \bm{b} & D
\end{bmatrix}^{-1} = \begin{bmatrix}
    c & -c\bm{q}^\top \\
    -c\bm{q} & D^{-1} + c\bm{q}\bm{q}^\top
\end{bmatrix}\,, \qquad c = \frac{1}{a-\bm{b}^\top D^{-1} \bm{b}}\,, \quad \bm{q} = D^{-1}\bm{b}$$ we have a formula for $\Gamma^*$ in terms of $G_\mathrm{ext}(0)$ as \begin{align*}
    \frac{\sqrt{d}}{\bm{e}_1^\top G_\mathrm{ext}(0)\bm{e}_1}G_\mathrm{ext}(0)\bm{e}_1 = \begin{bmatrix}
        \sqrt{d} \\
        \vecop(\Gamma)^*
    \end{bmatrix}
\end{align*}
Using our above formulas, we also have \begin{align*}
    \frac{\sqrt{d}}{\bm{e}_1^\top \G_\mathrm{ext}(0)\bm{e}_1}\G_\mathrm{ext}(0)\bm{e}_1 &= \begin{bmatrix}
        \sqrt{d} \\
        \qty(I_d\otimes B_\mathrm{tr} + \frac{1}{d}\Phi + \tilde{\lambda}I_{d(d+1)})^{-1}\vecop\qty(\begin{bmatrix}
            \Rtr + \frac{k_0}{\alpha}\tr[\Rtr]I_d & \rho_1\btr
        \end{bmatrix})
    \end{bmatrix}
\end{align*}
and so $$\vecop(\Gamma^*) \simeq \vecop(\Gamma_\mathrm{de}) \equiv \qty(I_d\otimes B_\mathrm{tr} + \frac{1}{d}\Phi + \tilde{\lambda}I_{d(d+1)})^{-1}\vecop\qty(\begin{bmatrix}
            \Rtr + \frac{k_0}{\alpha}\tr[\Rtr]I_d & \rho_1\btr
        \end{bmatrix}) = F_\Phi \bm{g}$$
for \begin{align*}
    F_\Phi &= \qty(I_d\otimes B_\mathrm{tr} + \frac{1}{d}\Phi + \tilde{\lambda}I_{d(d+1)})^{-1} \\
    \bm{g} &= \vecop\qty(\begin{bmatrix}
            \Rtr + \frac{k_0}{\alpha}\tr[\Rtr]I_d & \rho_1\btr
        \end{bmatrix}).
\end{align*}
This immediately gives the linear term in the MSE error expression as $$\tr[\Gamma^* A^\top] = \frac{1}{d}\tr(\Gamma^* A^\top) = \frac{1}{d}\tr(\vecop(\Gamma^*)\vecop(A)^\top) \simeq \frac{1}{d}\tr(\vecop(\Gamma_\mathrm{de})\vecop(A)^\top)$$ where we will specifically use 
$\vecop(A) = \qty(1+\phi_2)1_+$ from Lemma \ref{thm:population_icl_risk_CONSTANT}.

For the quadratic terms we need to be a bit more careful, as the correct object to work with for high-dimensional equivalence is technically the components of $\G$ and not this linear representation of $\vecop(\Gamma_\mathrm{de}).$ This is what the parameterisation $\pi B_\mathrm{ext}$ is for. Take $$B_\mathrm{ext} = \begin{bmatrix}
    0 & 0 \\
    0 & \Pi 
\end{bmatrix}$$ for some $\Pi \in \R^{d(d+1)\times d(d+1)}.$ Then $$\frac{\text{d}}{\text{d}\pi}\frac{1}{c(\pi)}(\pi=0) = \frac{1}{d}\vecop(\Gamma^*)^\top \Pi \vecop(\Gamma^*)$$ where $c(\pi) = \bm{e}_1^\top G_\mathrm{ext}(\pi)\bm{e}_1$. We can safely replace $$c(\pi) = \bm{e}_1^\top G_\mathrm{ext}(\pi)\bm{e}_1 \simeq \bm{e}_1\G_\mathrm{ext}(\pi)\bm{e}_1.$$ 

By Schur complement on $\G_\mathrm{ext}(\pi)$ we have \begin{align*}
    \frac{1}{\bm{e}_1\G_\mathrm{ext}(\pi)\bm{e}_1} &= \frac{\tau}{1+\chi_\pi}\rho_1+\tau\lambda - \frac{1}{d}\frac{\tau^2}{(1+\chi_\pi)^2}\bm{g}^\top\qty(\frac{\tau}{1+\chi_\pi}\qty(I_d\otimes B_\mathrm{tr}+\frac{1}{d}\Phi) + \pi\Pi + \tau\lambda I)^{-1}\bm{g} \\
    &=\frac{\tau}{1+\chi_\pi}\rho_1+\tau\lambda - \frac{1}{d}\frac{\tau}{1+\chi_\pi}\bm{g}^\top \qty(I_d\otimes B_\mathrm{tr} +\frac{1}{d}\Phi + \pi\frac{1+\chi_\pi}{\tau} \Pi +\lambda(1+\chi_\pi)I)^{-1}\bm{g}.
\end{align*}
Given the eventual MSE term we want from Lemma \ref{thm:population_icl_risk_CONSTANT}, we choose $$\Pi = I_d\otimes B + \frac{1}{d}\Psi$$ for 
$$\Psi = \begin{bmatrix}
    1_+ & \vecop\qty(\begin{bmatrix} I_d & \bm{0}\end{bmatrix})
\end{bmatrix} \begin{bmatrix}
        \phi_1 & \phi_2 \\
        \phi_2 & 0 
    \end{bmatrix} \begin{bmatrix}
        1_+^\top \\
        \vecop\qty(\begin{bmatrix} I_d & \bm{0}\end{bmatrix})^\top
    \end{bmatrix} = (\phi_1 + 2\phi_2)1_+1_+^\top. $$
Using the same approximation as before, we will take $$\left[\G_\mathrm{ext}\right]_{\setminus 0} = I_d\otimes \qty(\frac{\tau}{1+\chi_\pi}B_\mathrm{tr} + \pi B + \tau\lambda I)^{-1} + \frac{1}{d}\text{ low rank terms from $\Phi$ and $\Psi$}$$ and approximate $\chi_\pi$ as $$\chi_\pi = \tr\qty[\qty(\frac{\tau}{1+\chi_\pi}B_\mathrm{tr} + \pi B + \tau\lambda I)^{-1}B_\mathrm{tr}].$$ As before, we find that $$\frac{\tau\chi_0'}{(1+\chi_0)^2} = \frac{\tr[F_0BF_0B_\mathrm{tr}]}{\tr[F_0B_\mathrm{tr}F_0B_\mathrm{tr}]-\tau}\,, \qquad F_0 = (B_\mathrm{tr} + \tilde{\lambda}I)^{-1}.$$
Thus, \begin{align*}
    \frac{1}{d}\vecop(\Gamma^*)^\top \Pi \vecop(\Gamma^*) \simeq \frac{1}{d}\vecop(\Gamma_\mathrm{de})^\top \Pi \vecop(\Gamma_\mathrm{de}) - \frac{\tau \chi_0'}{(1+\chi_0)^2}\qty(\rho_1 - \frac{1}{d}\bm{g}^\top\vecop(\Gamma_\mathrm{de}) - \tilde{\lambda}\frac{1}{d}\vecop(\Gamma_\mathrm{de})^\top\vecop(\Gamma_\mathrm{de})).
\end{align*}

Note that \begin{align*}
    \vecop(\Gamma_\mathrm{de}) &= (I_d\otimes F_0)\bm{g} - \frac{1}{d}\Phi_C\bm{g} \\
    \Pi &= I_d \otimes B_\mathrm{test} + \frac{1}{d}\Psi
\end{align*} where \begin{align*}
    \Phi_C &= \qty((I_d \otimes F_0)T)\begin{bmatrix}
    \M & 1+ \phi_2(1-\sigma \M) \\
    1+ \phi_2(1-\sigma \M) & -\phi_1 + \phi_2^2(1-\sigma + \sigma^2\M)
\end{bmatrix}^{-1}\qty((I_d\otimes F_0)T)^\top \\
&= (I_d\otimes F_0)\begin{bmatrix}
    1_+ &\phi_2 r+ \phi_3 \mu
\end{bmatrix} S\begin{bmatrix}
    1_+^\top \\
    \phi_2 r^\top + \phi_3\mu^\top
\end{bmatrix}(I_d\otimes F_0)
\end{align*}
This comes from expanding \begin{align*}
F_0 &= \begin{bmatrix}
    F & \bm{f} \\
    \bm{f}^\top & f
\end{bmatrix} \\
\frac{1}{d}\begin{bmatrix}
        1_+^\top \\
        r^\top \\
        \mu^\top 
    \end{bmatrix}\begin{bmatrix}
        1_+ & r & \mu
    \end{bmatrix} &= \begin{bmatrix}
        \tr[F] & \tr[F\Rtr] &  \tr[\bm{f}\btr^\top] \\
        \tr[F\Rtr] & \tr[F\Rtr^2] & \tr[F\bm{f}\btr^\top ] \\
        \tr[\bm{f}\btr^\top] & \tr[F\bm{f}\btr^\top ] & f\tr[\btr\btr^\top]
    \end{bmatrix} \\
    &\simeq \begin{bmatrix}
        \M_\kappa(\sigma) & 1-\sigma \M_\kappa(\sigma) & 0 \\
        1- \sigma\M_\kappa(\sigma) & 1-\sigma +\sigma^2\M_\kappa(\sigma) & 0 \\
        0 & 0 & 0
    \end{bmatrix}\,, \qquad \sigma \equiv \frac{\rho_2}{\alpha} + \tilde{\lambda}.
\end{align*} 
As we can see, all the $\btr,\bm{f}$ contributions are negligible compared to the resolvent $(\Rtr + \sigma I_d)^{-1}$ contributions. Using similar intuition, in the expansions of the linear and quadratic error terms, we will neglect the $\rho_1 \btr$ component of $\bm{g}$, as well as $\mu = \vecop(\btr\bm{e}_{d+1}^\top)$ in $\Phi_C$.

\paragraph{Linear error term} Following this, and writing the matrix component of $\bm{g}$ as $$\tilde{R} = \Rtr + \frac{k_0}{\alpha}\tr[\Rtr]I_d$$ 
we have \begin{align*}
    \frac{1}{d}\vecop(A)^\top \vecop(\Gamma_\mathrm{de}) &= \qty(1+\phi_2)\frac{1}{d}\qty(1_+^\top (I\otimes F)\bm{g} - \frac{1}{d}1_+^\top\Phi_C\bm{g}) \\
    &\simeq \qty(1+\phi_2)\qty(\tr[F \tilde{R}] - \begin{bmatrix}
        \tr[F] & \phi_2\tr[F\Rtr]
    \end{bmatrix}S\begin{bmatrix}
        \tr[F\tilde{R}] \\
        \phi_2 \tr[\Rtr F\tilde{R}]
    \end{bmatrix})
\end{align*}

\paragraph{Quadratic error term} Recall $$B = \begin{bmatrix}
    \qty(1+\frac{\rho_2}{\alpha})I_d & \bm{0} \\
    \bm{0}^\top & \rho_1^2
\end{bmatrix}\,,\qquad \Psi = (\phi_1+2\phi_2)1_+1_+^\top.$$
Have 
$$\frac{1}{d}\bm{g}^\top \vecop(\Gamma_\mathrm{de}) \approx \tr[\tilde{R}F\tilde{R}] - \begin{bmatrix}
        \tr[\tilde{R}F] & \phi_2\tr[\tilde{R}F\Rtr]
    \end{bmatrix}S\begin{bmatrix}
        \tr[F\tilde{R}] \\
        \phi_2 \tr[\Rtr F \tilde{R}]
    \end{bmatrix} $$
\begin{align*}\frac{1}{d}\vecop(\Gamma_\mathrm{de})^\top \vecop(\Gamma_\mathrm{de}) &\approx \tr[\tilde{R}F^2\tilde{R}] - 2\begin{bmatrix}
        \tr[\tilde{R}F] &
        \phi_2 \tr[\tilde{R}F\Rtr] 
    \end{bmatrix}S\begin{bmatrix}
        \tr[F^2\tilde{R}] \\
        \phi_2 \tr[\Rtr F^2 \tilde{R}]
    \end{bmatrix} \\
    &\quad+\; \begin{bmatrix}
        \tr[\tilde{R}F] & \phi_2\tr[\tilde{R} F \Rtr]
    \end{bmatrix}S^\top \begin{bmatrix}
        \tr[F^2] & \phi_2 \tr[F^2\Rtr]\\
        \phi_2\tr[F^2\Rtr] & \phi_2^2\tr[\Rtr F^2 \Rtr]
    \end{bmatrix}S \begin{bmatrix}
        \tr[F\tilde{R}]\\
        \phi_2\tr[\Rtr F \tilde{R}]
    \end{bmatrix} 
\end{align*}
\begin{align*}\frac{1}{d}\vecop(\Gamma_\mathrm{de})^\top (I_d\otimes B) \vecop(\Gamma_\mathrm{de}) &\approx \qty(1+\frac{\rho_2}{\alpha}) \cdot \frac{1}{d}\vecop(\Gamma_\mathrm{de})^\top \vecop(\Gamma_\mathrm{de}) \\
    \frac{1}{\phi_1+2\phi_2}\frac{1}{d^2}\vecop(\Gamma_\mathrm{de})^\top \Psi \vecop(\Gamma_\mathrm{de}) &\approx \tr[\tilde{R}F]^2 - 2\tr[\tilde{R}F]\begin{bmatrix}
        \tr[F] &
        \phi_2\tr[\Rtr F]
    \end{bmatrix}S\begin{bmatrix}
        \tr[F\tilde{R}] \\
        \phi_2\tr[\Rtr F \tilde{R}]
    \end{bmatrix} \\
    &\quad+\; \qty(\begin{bmatrix}
        \tr[F] & \phi_2\tr[F\Rtr]
    \end{bmatrix}S\begin{bmatrix}
        \tr[F\tilde{R}]\\
        \phi_2\tr[\Rtr F \tilde{R}]
    \end{bmatrix})^2.
\end{align*}
The final step is to remember that $F = (\Rtr + \sigma)^{-1}$ and $\tilde{R} = \Rtr + k_0I_d/\alpha$ and so we have the following dictionary of terms \begin{align*}
    \tr[F] &\simeq \M_\kappa(\sigma) \\
    \tr[F\Rtr] &\simeq 1-\sigma\M_\kappa(\sigma)\\
    \tr[F\tilde{R}] &\simeq 1-\sigma\M_\kappa(\sigma)+\frac{k_0}{\alpha}\M_\kappa(\sigma)\\
    \tr[F\Rtr \tilde{R}] &\simeq \frac{k_0}{\alpha}(1-\sigma\M_\kappa(\sigma)) + 1-\sigma+\sigma^2\M_\kappa(\sigma)\\
    \tr[F \tilde{R}\tilde{R}] &\simeq \qty(\frac{k_0}{\alpha})^2\M_\kappa(\sigma) + 2\frac{k_0}{\alpha}(1-\sigma\M_\kappa(\sigma))) + 1-\sigma+\sigma^2\M_\kappa(\sigma)\\
    \tr[F^2] &\simeq -\M'_\kappa(\sigma)\\
    \tr[F^2 \Rtr] &\simeq \M_\kappa(\sigma) + \sigma \M'_\kappa(\sigma)\\
    \tr[F^2\tilde{R}] &\simeq \M_\kappa(\sigma) + \sigma \M'_\kappa(\sigma) - \frac{k_0}{\alpha}\M'_\kappa(\sigma)\\
    \tr[F^2 \Rtr \Rtr] &\simeq 1-2\sigma\M_\kappa(\sigma) - \sigma^2\M'_\kappa(\sigma)  \\
    \tr[F^2\Rtr \tilde{R}] &\simeq 1-2\sigma\M_\kappa(\sigma) - \sigma^2\M'_\kappa(\sigma) + \frac{k_0}{\alpha}(\M_\kappa(\sigma) + \sigma\M'_\kappa(\sigma)) \\
    \tr[F^2\tilde{R}\tilde{R}] &\simeq 1-2\sigma\M_\kappa(\sigma) - \sigma^2\M'_\kappa(\sigma) + 2\frac{k_0}{\alpha}(\M_\kappa(\sigma) + \sigma\M'_\kappa(\sigma)) - \qty(\frac{k_0}{\alpha})^2\M'_\kappa(\sigma)
\end{align*}

So finally, if we use some shorthand $$\tilde{\sigma} = \sigma - \frac{k_0}{\alpha}\,,\quad \bm{m}_1 = \begin{bmatrix}
    \M \\
    \phi_2(1-\sigma \M)
\end{bmatrix}\,,\quad \bm{m}_2 = \begin{bmatrix}
    1-\tilde{\sigma}\M \\
    \phi_2(1 - \tilde{\sigma} + \sigma\tilde{\sigma}\M)
\end{bmatrix}\,, \quad \bm{m}_3 = \begin{bmatrix}
    \M + \tilde{\sigma}\M' \\
    \phi_2(1-\sigma\M - \tilde{\sigma}\M -\sigma\tilde{\sigma}\M')
\end{bmatrix}$$
$$M = \begin{bmatrix}
    -\M' & \phi_2(\M+ \sigma\M')\\
    \phi_2(\M+ \sigma\M') & \phi_2^2(1-2\sigma\M - \sigma^2\M')
\end{bmatrix}$$
and so \begin{align*}
    \frac{1}{d}\vecop(A)^\top \vecop(\Gamma_\mathrm{de}) &\simeq (1+\phi_2)\qty(1-\tilde{\sigma}\M - \bm{m}_1^\top S \bm{m}_2) \\
    \frac{1}{d}\bm{g}^\top\vecop(\Gamma_\mathrm{de}) &\simeq 1+\sigma -2\tilde{\sigma}+ \tilde{\sigma}^2\M -\bm{m}_2^\top S\bm{m}_2 \\
    \frac{1}{d}\vecop(\Gamma_\mathrm{de})^\top \vecop(\Gamma_\mathrm{de}) &\simeq 1-2\tilde{\sigma}\M - \tilde{\sigma}^2\M' - 2\bm{m}_2^\top S\bm{m}_3 + \bm{m}_2^\top S^\top MS\bm{m}_2 \\
    \frac{1}{d^2}\vecop(\Gamma_\mathrm{de})^\top \Psi \vecop(\Gamma_\mathrm{de}) &\simeq (\phi_1+2\phi_2)\qty((1-\tilde{\sigma}\M) -\bm{m}_1^\top S \bm{m}_2)^2
\end{align*} with $$c = \frac{\tau\chi_0'}{(1+\chi_0)^2} \simeq (1+\frac{\rho_2}{\alpha})\frac{\M + \tilde{\lambda}\M'}{1-2\tilde{\lambda}\M - \tilde{\lambda}^2\M' - \tau}.$$

So finally, we have \begin{align*}
    e_\mathrm{ICL} &= (1-c)\rho_1 - 2(1+\phi_2)(1-\tilde{\sigma}\M - \bm{m}_1^\top S \bm{m}_2) + (\phi_1+2\phi_2)(1-\tilde{\sigma}\M - \bm{m}_1S\bm{m}_2)^2 \\
    &\quad+\; \qty(1+\frac{\rho_2}{\alpha}+c\tilde{\lambda})(1-2\tilde{\sigma}\M - \tilde{\sigma}^2\M' - 2\bm{m}_2^\top S\bm{m}_3 + \bm{m}_2^\top S^\top MS\bm{m}_2) \\
    &\quad+\; c(1+\sigma -2\tilde{\sigma}+ \tilde{\sigma}^2\M -\bm{m}_2^\top S\bm{m}_2) 
\end{align*} for 
$$\rho_1 \equiv 1 + \rho\,, \qquad \rho_2 \equiv \mathrm{tr}[K^2] + \rho $$
$$\sigma = \tilde{\lambda} + \frac{\rho_2}{\alpha}\,, \quad \tilde{\sigma} = \sigma - \frac{k_0}{\alpha}\,, \quad     \tilde{\lambda} \M_\kappa\qty(\tilde{\lambda} + \frac{\rho_2}{\alpha}) + \frac{\lambda\tau}{\tilde{\lambda}}= 1-\tau\,, \quad c =  (1+\frac{\rho_2}{\alpha})\frac{\M + \tilde{\lambda}\M'}{1-2\tilde{\lambda}\M - \tilde{\lambda}^2\M' - \tau}$$
$$\quad \bm{m}_1 = \begin{bmatrix}
    \M \\
    \phi_2(1-\sigma \M)
\end{bmatrix}\,,\quad \bm{m}_2 = \begin{bmatrix}
    1-\tilde{\sigma}\M \\
    \phi_2(1 - \tilde{\sigma} + \sigma\tilde{\sigma}\M)
\end{bmatrix}\,, \quad \bm{m}_3 = \begin{bmatrix}
    \M + \tilde{\sigma}\M' \\
    \phi_2(1-\sigma\M - \tilde{\sigma}\M -\sigma\tilde{\sigma}\M')
\end{bmatrix}$$
$$M = \begin{bmatrix}
    -\M' & \phi_2(\M+ \sigma\M')\\
    \phi_2(\M+ \sigma\M') & \phi_2^2(1-2\sigma\M - \sigma^2\M')
\end{bmatrix}\,,\quad S = \begin{bmatrix}
    \M & 1+ \phi_2(1-\sigma \M) \\
    1+ \phi_2(1-\sigma \M) & -\phi_1 + \phi_2^2(1-\sigma + \sigma^2\M)
\end{bmatrix}^{-1}$$
$$\phi_1 \equiv \frac{1}{\alpha^2}(k_1 + \rho k_0)=\phi_1\,, \quad \phi_2 \equiv \frac{1}{\alpha}k_0=\phi_2$$

\noindent \textbf{Sanity check}. For $K = I_d$ have $\rho_1=\rho_2$, $\sigma = \tilde{\sigma}$, $\phi_1=\phi_2=\phi_1=\phi_2=0$, and all $\bm{m},S,M$ terms do not contribute. Get \begin{align*}
    e_\mathrm{ICL} &= 1+ \rho - 2(1-\sigma \M) + \qty(1+\frac{1+\rho}{\alpha})(1-2\sigma \M - \sigma^2 \M') \\
    &\quad-\; c(\rho + \sigma-\sigma^2\M-\tilde{\lambda}(1-2\sigma\M - \sigma^2\M')) \\
    &=  1+ \rho - 2(1-\sigma \M) + \qty(1+\frac{1+\rho}{\alpha})(1-2\sigma \M - \sigma^2 \M') + \qty(1+ \frac{1+\rho}{\alpha})(\M + \tilde{\lambda}\M')q  \\
    &= \rho + q\M + (\tilde{\lambda}q-\sigma^2)\M' + \frac{1+\rho}{\alpha}(1-(q-2\sigma)\M + (\tilde{\lambda}q - \sigma^2)\M')
\end{align*} as previously in \cite{letey2025pretrain}. 

\subsection{Dependence on query terms} This formula can be cleaned up a fair bit by gathering terms that depend on $k_0, k_1$ \textit{i.e.}, the query correlation terms. First we write all the terms depending on $c$ as \begin{align*}
    &(-c)\qty(\rho_1 - \tilde{\lambda}\qty(1-2\tilde{\sigma}\M - \tilde{\sigma}^2\M' - 2\bm{m}_2^\top S\bm{m}_3 + \bm{m}_2^\top S^\top MS\bm{m}_2) - \qty(1+\sigma -2\tilde{\sigma}+ \tilde{\sigma}^2\M -\bm{m}_2^\top S\bm{m}_2)) \\
    &\quad=\; \qty(1+\frac{\rho_2}{\alpha})\frac{\M + \tilde{\lambda}\M'}{\tau - (1-2\tilde{\lambda}\M - \tilde{\lambda}^2\M')}\qty(\rho + \sigma - \sigma^2\M - \tilde{\lambda}(1-2\sigma\M - \sigma^2\M')) \\
    &\quad\quad-\;\qty(1+\frac{\rho_2}{\alpha})\frac{\M + \tilde{\lambda}\M'}{\tau - (1-2\tilde{\lambda}\M - \tilde{\lambda}^2\M')}\qty(2\frac{k_0}{\alpha}\qty(\tilde{\lambda}(\M + \sigma\M')+(1-2\sigma\M)) + \frac{k_0^2}{\alpha^2}\qty(\M-\tilde{\lambda}\M')) \\
    &\quad\quad+\;\qty(1+\frac{\rho_2}{\alpha})\frac{\M + \tilde{\lambda}\M'}{\tau - (1-2\tilde{\lambda}\M - \tilde{\lambda}^2\M')}\qty(\bm{m}_2^\top S\bm{m}_2 - \tilde{\lambda}\qty(- 2\bm{m}_2^\top S\bm{m}_3 + \bm{m}_2^\top S^\top MS\bm{m}_2)) \\
    &\quad=\; \qty(1+\frac{\rho_2}{\alpha})(\M + \tilde{\lambda}\M')(q_\text{old} + q_\text{query})
\end{align*}
and the remaining ICL error terms can be expressed as \begin{align*}
    &\rho_1 -2(1-\sigma\M) + \qty(1+\frac{\rho_2}{\alpha})(1-2\sigma\M - \sigma^2\M') \\
    &\quad-\; \frac{k_0}{\alpha}\qty(2\M -2\qty(1+\frac{\rho_2}{\alpha})(\M +\sigma\M')) - \frac{k_0^2}{\alpha^2}\qty(1+\frac{\rho_2}{\alpha})\M' \\
    &\quad-\; 2\phi_2(1-\tilde{\sigma}\M) + 2(1+\phi_2)\bm{m}_1^\top S\bm{m}_2 \\
    &\quad+\; \qty(1+\frac{\rho_2}{\alpha})(\bm{m}_2^\top S^\top MS\bm{m}_2 - 2\bm{m}_2^\top S \bm{m}_3) \\
    &\quad+\; (\phi_1+2\phi_2)(1-\tilde{\sigma}\M - \bm{m}_1^\top S\bm{m_2})^2.
\end{align*}
Thus we can write \begin{align*}
    e_\text{ICL}(k_0,k_1) &= e_\text{ICL}(\text{independent query}) -2\phi_2
+ q_{\mathrm{query}}(\M+\tilde{\lambda}\M')
+ \phi_2(2\sigma-\phi_2)\M' \\
&\quad+\; \bm{m}_2^\top S^\top M\,S\bm{m}_2
- 2\bm{m}_2^\top S \bm{m}_3
+ 2\phi_2\tilde{\sigma}\M + 2(1+\phi_2)\bm{m}_1^\top S \bm{m}_2
\\
&\quad+\; (\phi_1+2\phi_2)\left(1-\tilde{\sigma}\M-\bm{m}_1^\top S \bm{m}_2\right)^2 \\
&\quad
+ \frac{\rho_2}{\alpha}\Bigl[
q_{\mathrm{query}}(\M+\tilde{\lambda}\M')
+ \phi_2(2\sigma-\phi_2)\M'
+ \bm{m}_2^\top S^\top M\,S\bm{m}_2
- 2\bm{m}_2^\top S \bm{m}_3
+ 2\phi_2\M
\Bigr].
\end{align*}

\end{document}